\def\OpenSourceVersion{1}
\documentclass[letterpaper]{article} 
\ifdefined\OpenSourceVersion
\usepackage[preprint]{aaai2027}  
\copyrighttext{Preprint}
\else
\usepackage[submission]{aaai2027}  
\fi
\usepackage[hyphens]{url}  
\usepackage{graphicx} 
\usepackage{natbib}  
\usepackage{caption} 
\usepackage{algorithm}
\usepackage{algpseudocode}

\usepackage{newfloat}
\usepackage{listings}
\DeclareCaptionStyle{ruled}{labelfont=normalfont,labelsep=colon,strut=off} 
\floatstyle{ruled}
\newfloat{listing}{tb}{lst}{}
\floatname{listing}{Listing}

\usepackage{booktabs}
\usepackage{multirow}
\usepackage{amsmath,amssymb,amsthm}
\usepackage{xr}
\ifdefined\OpenSourceVersion
\else
\fi

\newtheorem{proposition}{Proposition}
\newtheorem{lemma}[proposition]{Lemma}

\newcommand{\method}{R$^2$-OPD}
\newcommand{\methodname}{Reasoning-Progress-Aware Reward Filtering for On-Policy Distillation}

\title{Beyond Imitation: Filtering On-Policy Distillation by Reasoning Progress}

\ifdefined\OpenSourceVersion
\author{
    Chen Yang\textsuperscript{\rm 1},
    Haiyuan Wan\textsuperscript{\rm 2},
    Rengrong Xiong\textsuperscript{\rm 3},
    Yize Chen\textsuperscript{\rm 4}\corresponding,
    Danny H.K. Tsang \textsuperscript{\rm 1}
}

\affiliations{
    \textsuperscript{\rm 1}The Hong Kong University of Science and Technology (Guangzhou)\\
    \textsuperscript{\rm 2}Tsinghua University\\
    \textsuperscript{\rm 3}Zhejiang University\\
    \textsuperscript{\rm 4}University of Alberta\\
    cyang529@connect.hkust-gz.edu.cn, yize9@ualberta.ca\\
}
\else
\author{Anonymous Submission}
\affiliations{Anonymous Institution}
\fi

\begin{document}
\maketitle

\begin{abstract}
On-policy distillation (OPD) has emerged as an effective framework for post-training language models by pairing student-generated trajectories with dense token-level supervision from a teacher. However, OPD implicitly assumes that teacher-derived rewards are an appropriate proxy for reasoning progress, and therefore treats all teacher feedback equally during policy optimization. While in practice, this assumption does not always hold. We observe that teacher-derived rewards often conflict with genuine reasoning progress, as reasoning steps with clear reasoning advancement may still receive lower distillation rewards, simply due to deviation from teacher's outputs. To address this mismatch, we propose {\methodname{} (\method{})}, which constructs two within-trajectory rankings of reasoning spans, one from teacher-derived rewards and the other from independently estimated progress reward. Distillation rewards are selectively suppressed whenever the two rankings disagree, reducing supervision that conflicts with reasoning progress while preserving effective teacher guidance. Our approach shows consistent improvement over standard OPD especially regarding reasoning performances.
\end{abstract}


\section{Introduction}
Knowledge distillation transfers the predictive behavior of a capable teacher model to a naive student model, offering a practical route to improving language models without the full expense of training or deploying the teacher~\cite{hinton2015distilling,kim2016sequence}. For autoregressive generation, conventional distillation is commonly performed on static teacher-generated datasets, leading to an exposure bias between the training trajectories and inference trajectories ~\cite{bengio2015scheduled}. On-policy distillation (OPD) addresses this mismatch by taking student-generated responses as training inputs and querying the teacher at the specific states the student encounters. Combined with token-level distribution matching and policy optimization, OPD provides dense supervision across the entire response, serving as a promising paradigm for post-training reasoning models and model merging \cite{song2026survey, xu2026deepseek}.

Despite these advantages, token-level teacher supervision does not directly \emph{measure whether a reasoning span advances the solution.} At every decoding step, OPD rewards the student for staying close to the teacher distribution, implicitly treating teacher similarity as a proxy for reasoning and answer  quality~\cite{heo2026policy}. This proxy can be unreliable because a student-generated span may increase the probability of arriving at the correct answer even if it departs from the teacher trajectory \cite{jiang2026trajectoryrefineddistillation}. Conversely, a teacher-like span does not necessarily make meaningful, concise progress toward the ultimate solution. Consequently, a constructive reasoning step may receive unfavorable distillation guidance simply for deviating from the teacher. Applying such uncalibrated signals uniformly risks suppressing valid reasoning pathways rather than effectively transferring functional knowledge~\cite{plyusov2026trust}.

This observation motivates a distinction between \emph{distillation compatibility} and \emph{reasoning progress}. While the former measures agreement with the teacher, the latter measures how an intermediate state changes the likelihood of solving the task. Process rewards provide a natural teacher-independent surrogate for reasoning progress by sampling continuations from successive reasoning states and comparing their solve probabilities to estimate the marginal contribution of each span \cite{jia2025we,lightman2024let,Uesato2022SolvingMW}. However, directly combining raw process rewards with token-level divergence can be ineffective, as these two signals operate on different scales and exhibit fine-grained noise. The former inherits Monte Carlo noise from a finite number of continuations, while the latter can fluctuate with local lexical choices and teacher uncertainty.

In this work, we propose \textbf{\methodname{} (\method{})}, a novel reward-filtering framework for reasoning-oriented OPD. Across consecutive reasoning spans, \method{} first merges adjacent reasoning spans with sign-consistent process rewards. Such a procedure removes dependence on noisy internal boundary estimates and produces more stable units for reward calibration. It then aggregates token-level divergence within each merged span, and compares the relative ordering induced by reasoning progress with that induced by teacher divergence. When a higher-progress span is penalized more strongly for departing from the teacher, the two signals exhibit a local ranking conflict. Rather than replacing the OPD reward or adding process reward as a new optimization objective, \method{} utilizes this conflict only as a reliability test and masks the distillation rewards of the most inconsistent spans. In this way, the method preserves teacher guidance where it agrees with reasoning progress, while reducing supervision that may discourage productive reasoning. This approach not only makes teacher guidance more effective, but also significantly boosts OPD performance in complex reasoning scenarios. Our contributions are as follows:
\begin{itemize}
    \item We identify a common failure mode of reasoning-oriented OPD and show that uniformly applied distillation rewards may provide misleading supervision. By introducing \method{}, it adopts independently estimated process rewards to construct a teacher-independent reference, detects local progress--distillation ranking conflicts, and selectively masks unreliable segment-level supervision. 
    \item We develop sign-consistent process-reward merging and segment-level divergence averaging, and provide theoretical results characterizing the cancellation of internal boundary-estimation errors and the reduction of local divergence variance under weak dependence.
    \item We empirically demonstrate that filtering progress-conflicting distillation signals consistently improves reasoning performance over standard OPD.
\end{itemize}

\section{Preliminaries}
\label{sec:prelim}
\subsection{On-Policy Distillation}
Let $x$ and $y=(y_1,\dots,y_T)\sim\pi_S(\cdot\mid x)$ denote a prompt and a response sampled on-policy from the student model, respectively. Let $h_t=(x,y_{<t})$ denote the decoding context at step $t$. 
In this work, reverse Kullback--Leibler (KL) divergence-based On-Policy Distillation (OPD) is taken into investigation. Specifically, OPD optimizes the student over its self-generated trajectories by minimizing the reverse KL divergence between the student and teacher token distributions at each decoding step \cite{ICLR2024_5be69a58,gu2024minillm}. 
Given teacher policy $\pi_T$, evaluating this reverse KL divergence over the full vocabulary effectively weighs each token by student policy $\pi_S$, leading to the following objective:
\begin{small}
\begin{equation}
\begin{aligned}
\mathcal{L}_{\mathrm{OPD}}
&=
\mathbb{E}_{\substack{x, y\sim\pi_S(\cdot\mid x)}}
\left[
\frac{1}{T}\sum_{t=1}^{T}\ell_t^{\mathrm{KL}}
\right], \\
\ell_t^{\mathrm{KL}}
&=
D_{\mathrm{KL}}
\!\left(
\pi_S(\cdot\mid h_t)
\,\|\,\pi_T(\cdot\mid h_t)
\right) \\ &=
\sum_{v\in\mathcal{V}}
\pi_S(v\mid h_t)
\log
\frac{\pi_S(v\mid h_t)}
     {\pi_T(v\mid h_t)} .
\end{aligned}
\label{eq:opd-objective}
\end{equation}
\end{small}
Directly computing the KL divergence over the entire vocabulary $\mathcal{V}$ at every step is computationally prohibitive, and extremely low-probability tokens offer negligible supervisory signals. Consequently, practical OPD implementations approximate Eq.~\eqref{eq:opd-objective} using a support set $\mathcal{S}_t\subseteq\mathcal{V}$ of size $H$. The support can be constructed from the $H$ highest-probability tokens of the student, the teacher, or their union.

For the student-derived support used in our implementation, we define the normalized student weight
\begin{small}
\begin{equation}
w_{t,v}
=
\frac{\pi_S(v\mid h_t)}
{\sum_{u\in\mathcal{S}_t}\pi_S(u\mid h_t)},
\qquad v\in\mathcal{S}_t,
\label{eq:support-weight}
\end{equation}
\end{small}
And then we compute the support-restricted approximation to the token-level KL loss as
\begin{small}
\begin{equation}
\ell_t^{\mathrm{KL},\mathcal{S}}
=
\sum_{v\in\mathcal{S}_t}
w_{t,v}
\log
\frac{\pi_S(v\mid h_t)}
     {\pi_T(v\mid h_t)}.
\label{eq:support-loss}
\end{equation}
\end{small}

\subsection{Process Reward Estimation}
\label{sec:process-reward}

\begin{figure}
    \centering
    \includegraphics[width=\linewidth]{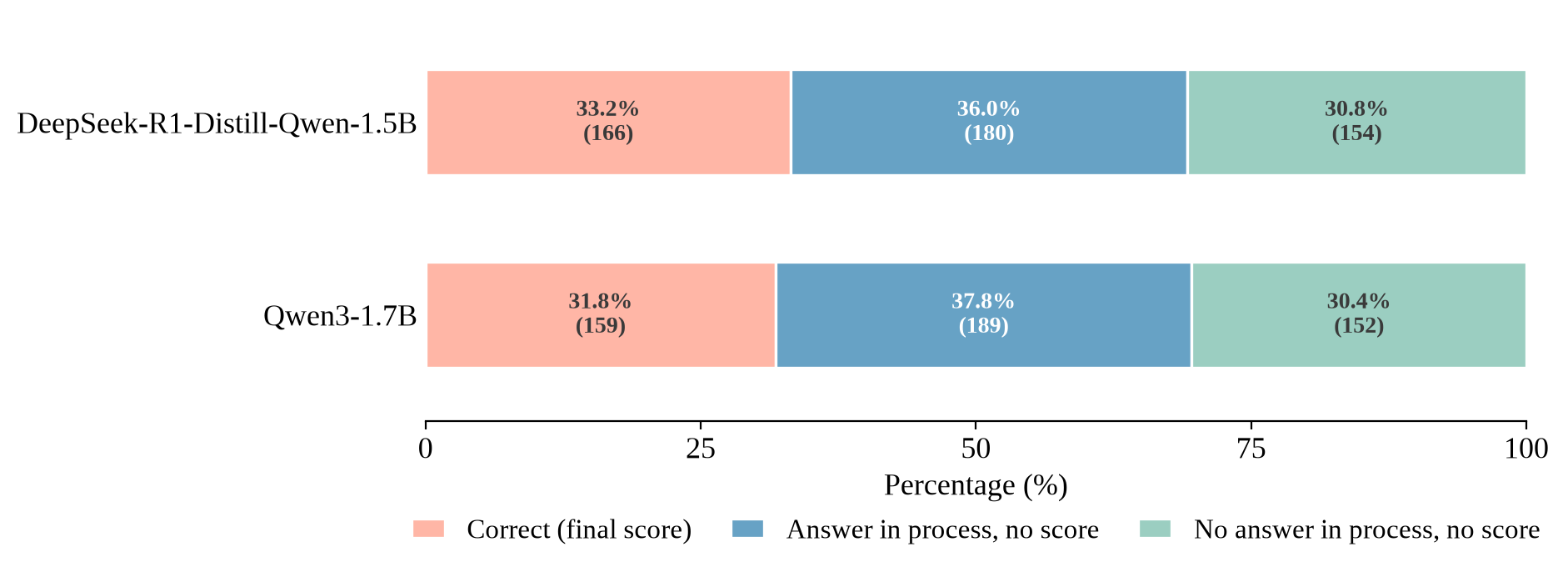}
    \caption{Distribution of evaluation outcomes for DeepSeek-R1-Distill-Qwen-1.5B and Qwen3-1.7B on DAPO (500 samples, max 7,168 response tokens).}
    \label{fig:barfigure}
\end{figure}
While the token-level supervision introduced above measures distributional agreement with the teacher, it fails to explicitly reflect the correctness of the final outcome. To bridge this gap, recent OPD methods for reasoning tasks have explored the incorporation of outcome-level correctness to calibrate or filter token-level teacher supervision~\cite{hou2026uni,akhondzadeh2026reward,zheng2026scope}.
However, outcome feedback indicates only whether a completed response is correct and provides limited information about which intermediate steps advance the solution, introduce an error, or recover from an earlier mistake. Furthermore, as illustrated in Figure \ref{fig:barfigure}, a third of the responses contain correct reasoning paths, but are truncated due to overthinking. Relying solely on outcome-based filtering thus fails to fully exploit such valuable trajectories.
Process rewards address this temporal credit-assignment problem by evaluating intermediate reasoning states or transitions, thereby providing a finer-grained account of how the likelihood of task success evolves along a trajectory \cite{lightman2024let,Uesato2022SolvingMW}.

This form of supervision has become useful across broader post-training workflows~\cite{cui2025process}. The process reward can guide reinforcement learning by assigning credit to individual reasoning steps, support rejection sampling or reranking among candidate solutions, and steer test-time search toward promising partial trajectories~\cite{Luo2024ImproveMR,zhang2024rest}. When manual step-level annotation is unavailable, prior work has constructed automatic process supervision from the success rate of sampled continuations, grounding the value of an intermediate state in its probability of eventually reaching a correct answer \cite{wang2024math,qu2025optimizing,setlur2025rewarding}. 

Following this intuition, we estimate the solve probability of intermediate reasoning states through on-policy rollouts, and design process rewards as the incremental change in solve probability across consecutive reasoning states.
Given a response $y = (y_1, \dots, y_T)$, we partition it into $M$ contiguous reasoning segments $\Sigma(y) = \{\sigma_1, \dots, \sigma_M\}$. 
Segment boundaries are determined in two stages. We first identify candidate boundary positions in $y$ by matching its tokens against a fixed lexicon $\mathcal{T}$ of self-reflective discourse markers. To prevent over-segmentation caused by consecutive discourse markers with minimal reasoning content, we accept a candidate boundary $b_{m-1}$ only when $N_{\mathrm{sent}}(y_{b_{m-1}+1:t}) \ge S_{\min}$. Here, $N_{\mathrm{sent}}(\cdot)$ measures the sentence count, and $S_{\min}$ specifies the minimum number of sentences required between successive accepted boundaries. The complete lexicon and matching procedure are provided in Appendix~\ref{app:segmentation}.

To quantify the reasoning progress at boundary state $p_m = (x, y_{1:b_m})$, we estimate its solve probability via $N_{\mathrm{eval}}$ on-policy Monte Carlo rollouts:
\begin{equation}
c_m^{(1)}, \dots, c_m^{(N_{\mathrm{eval}})} \;\sim\; \pi_S(\cdot \mid \bar{p}_m, \mathcal{I}_{\text{ans}}),
\end{equation}
where $\bar{p}_m$ is the format-standardized prefix and $\mathcal{I}_{\text{ans}}$ is an answer-eliciting prompt (Appendix~\ref{app:answer-instruction}).
For each prefix state  $p_m$, the corresponding solve probability is estimated by Monte Carlo sampling using the ground-truth answer $g$ and verifiable reward function $\mathcal{R}(\cdot, \cdot)$,
\begin{small}
    \begin{equation}
\hat{S}_m = 
\begin{cases} 
0, & m = 0, \\[2pt]
\frac{1}{N_{\mathrm{eval}}}\sum_{\ell=1}^{N_{\mathrm{eval}}} \mathcal{R}\big(p_m \mathbin{\Vert} c_m^{(\ell)},\, g\big), & 1 \le m \le M-1, \\[2pt]
\mathcal{R}(y, g), & m = M.
\end{cases}
\label{eq:solve-prob}
\end{equation}
\end{small}
Eq.~\eqref{eq:solve-prob} assigns zero credit to the initial state, and reuses the actual response's correctness for $m=M$, avoiding unnecessary generation passes for the terminal state.

The \textbf{process reward} $PR_m$ for segment $\sigma_m$ is defined as the marginal gain in solve probability across its boundary:
\begin{equation}
PR_m \;\triangleq\; \hat{S}_m - \hat{S}_{m-1}, \qquad m = 1,\dots,M.
\label{eq:process-reward}
\end{equation}
To avoid redundant computation, $PR_m$ is evaluated only when segmentation succeeds and a string-level check confirms $g \in y$. Otherwise, the trajectory is tagged as \emph{PR-unavailable} and passed through unfiltered. This procedure ensures \method{} never penalizes trajectories which lack explicit process signals.

\section{Reasoning-Progress-Aware Reward Filtering}
The token-level OPD reward measures how closely the student matches the teacher at each decoding state, but it does not indicate whether a reasoning span actually improves the student's ability to solve the problem. Conversely, terminal correctness provides only trajectory-level supervision and cannot attribute success or failure to individual reasoning spans. As a result, a reasoning span that genuinely increases the student's probability of solving the problem may still receive a low distillation reward simply because its token distribution differs from that of the teacher. To address this limitation, we use the process reward defined in Eq.~\eqref{eq:process-reward} as an independent reference when evaluating the reliability of the OPD supervision for each reasoning span. Importantly, the process reward neither replaces the OPD reward nor serves as an additional optimization objective. Instead, it provides a teacher-independent estimate of reasoning progress, allowing us to assess whether the distillation signal is aligned with the estimated contribution of each reasoning span. Our method consists of three stages, which are described in detail in the following subsections. This algorithm is illustrated in Figure \ref{fig:alg}. For clarity, we describe the segment-level operations for a single response and omit the response index unless multiple responses must be distinguished.

\paragraph{Noise Reduction via Sign-Consistent merging.}
\begin{figure}
    \centering
    \includegraphics[width=\linewidth]{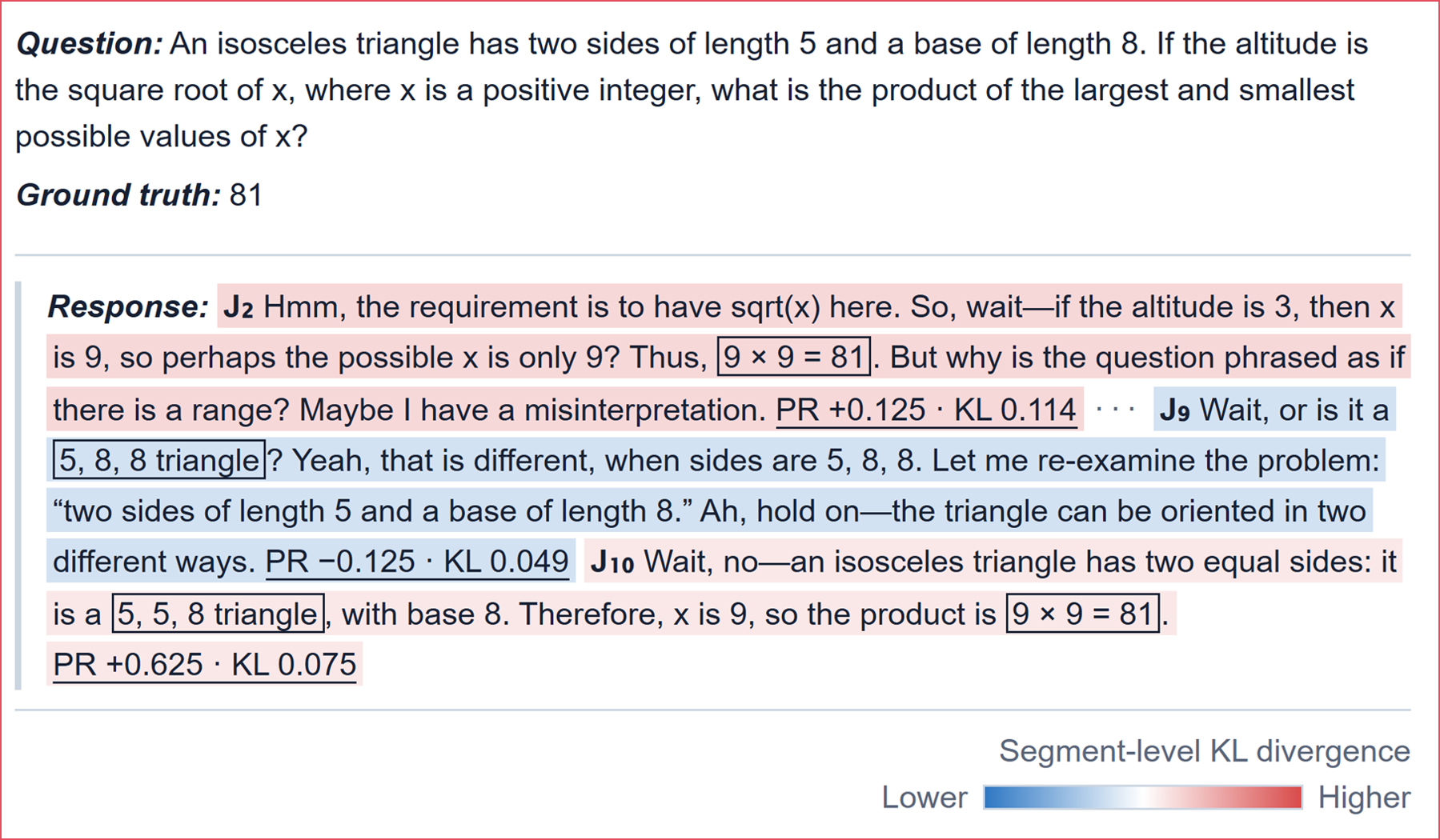}
    \caption{representative mismatch between reasoning progress and the OPD divergence signal. A reasoning regressive segment ($J_9$) receives a smaller KL divergence under Equation~\eqref{eq:segment-loss} than segments that derive or recover the correct solution, respectively($J_2$ and $J_{10}$).}
    \label{fig:voilation}
\end{figure}

In practice, both process rewards and token-level distillation signals can be very noisy at a fine granularity, posing challenges for policy optimization. Process rewards are estimated from a finite number of continuations, while token-level divergences can fluctuate substantially because of local lexical choices, teacher uncertainty, and the top-$k$ approximation. Directly comparing these signals over short reasoning spans may therefore frequently produce unstable and even conflicting guidance. To overcome this challenge, \method{} first aggregates adjacent spans that exhibit a consistent direction of estimated progress, obtaining coarser units on which both signals can be evaluated more reliably.

Specifically, because estimation with a small $N_{\mathrm{eval}}$ renders individual $PR_m$ noisy, we merge adjacent segments whose nonzero process rewards have the same sign. A zero-valued process reward is absorbed into the current run and does not create a new boundary; a new run begins only when the nonzero sign changes from positive to negative or vice versa. This partitions $\{1,\dots,M\}$ into $n$ maximal contiguous index sets $\mathcal{J}_1, \dots, \mathcal{J}_n$. Each index set $\mathcal{J}_j$ induces a merged segment $\tilde{\sigma}_j$ by concatenating its constituent segments, with aggregated process reward:
\begin{equation}
\widetilde{PR}_j \;\triangleq\; \sum_{m \in \mathcal{J}_j} PR_m, \qquad j = 1,\dots,n.
\label{eq:merged-reward}
\end{equation}
Because process rewards are defined as differences between consecutive solve-probability estimates, summing them over a merged run produces a telescoping sum. Consequently, the aggregated reward depends only on the estimates at the two endpoints rather than on those at the internal boundaries. We formalize this property below.

\begin{figure*}
    \centering
    \includegraphics[width=0.9\linewidth]{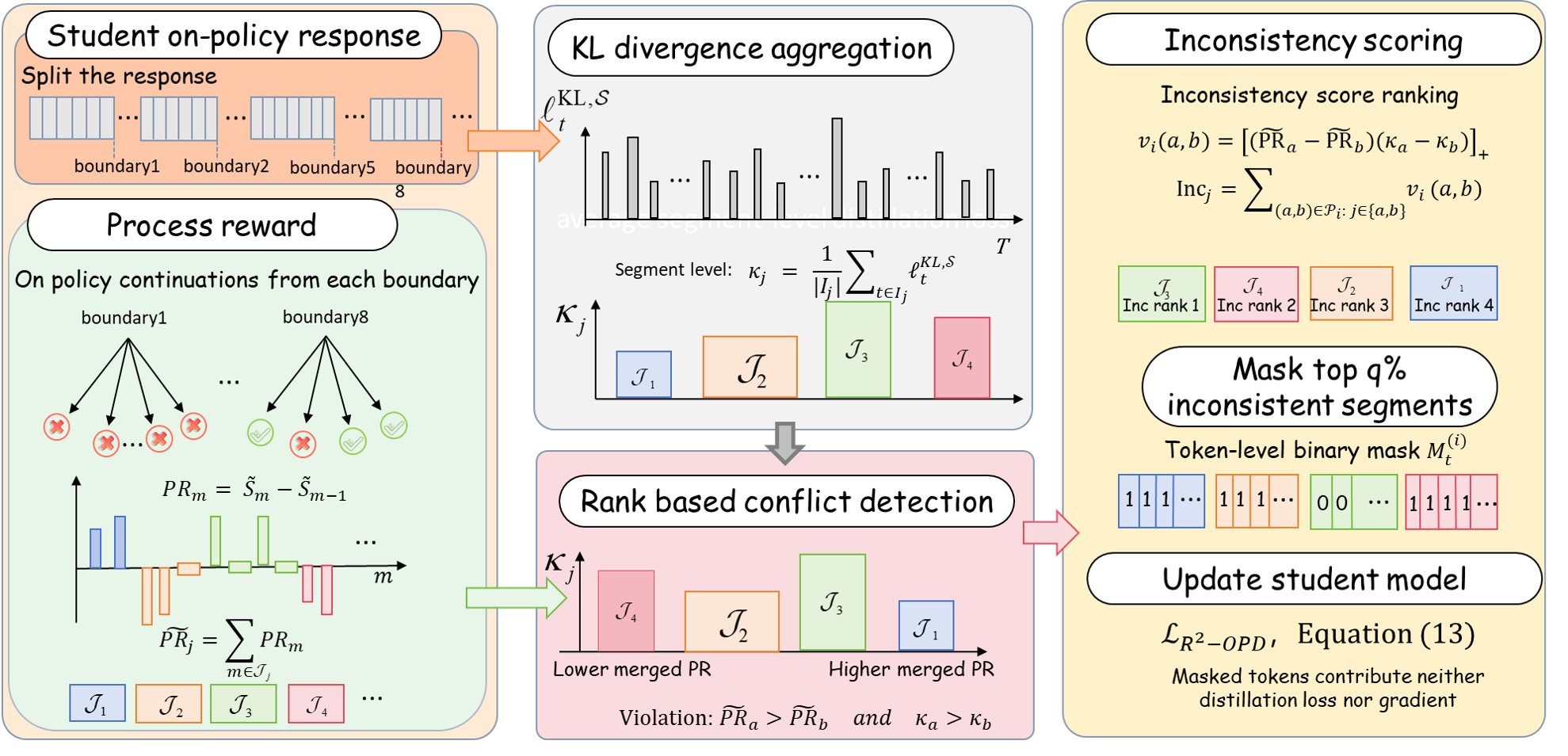}
    \caption{Overview of the proposed Reasoning-Progress-Aware Reward Filtering OPD. Adjacent reasoning spans with consistent progress signals are first merged. For each merged segment, we sum its process rewards and average its token-level KL divergences. Rank conflicts between the resulting segment-level signals are then used to identify unreliable spans, which are masked during the OPD policy update.}
    \label{fig:alg}
\end{figure*}

\begin{lemma}[Telescoping of Aggregated Process Rewards]
\label{lem:boundary-cancellation}
Let $\mathcal{J}_j=\{a,\dots,b\}$ be a merged run, and write each estimated solve probability as $\hat{S}_m=S_m+\epsilon_m$, where $S_m$ is the underlying solve probability and $\epsilon_m$ is its Monte Carlo estimation error. Then
\begin{equation}
\widetilde{PR}_j
=
(S_b-S_{a-1})+(\epsilon_b-\epsilon_{a-1}).
\label{eq:telescoping-pr}
\end{equation}
Hence, the aggregated reward is independent of the internal boundary errors $\epsilon_a,\dots,\epsilon_{b-1}$ and retains only the errors at the two endpoints.
\end{lemma}

\begin{proof}
Substituting $PR_m=\hat{S}_m-\hat{S}_{m-1}$ into Eq.~\eqref{eq:merged-reward} gives
\[
\sum_{m=a}^{b}PR_m
=
\sum_{m=a}^{b}(\hat{S}_m-\hat{S}_{m-1})
=
\hat{S}_b-\hat{S}_{a-1},
\]
because all internal terms telescope. Substituting $\hat{S}_m=S_m+\epsilon_m$ yields Eq.~\eqref{eq:telescoping-pr}.
\end{proof}

Restricting merging to sign-consistent runs avoids canceling progress signals that reflect genuine changes in reasoning direction. The resulting segments therefore capture progress across major turning points while eliminating dependence on noisy internal boundary estimates.

\subsection{Rank-Based Conflict Detection and Masking}
\label{sec:conflict-masking}
For each merged segment $j \ge 1$, let $\mathcal{I}_j \subseteq \{1,\dots,T\}$ denote the token positions belonging to $\tilde{\sigma}_j$. In order to compare token-level distillation feedback with segment-level process rewards at a common granularity while avoiding a systematic dependence on segment length, we average the token-level signals within each merged segment. We refer to the resulting mean as the average segment-level distillation loss $\kappa_j$:
\begin{equation}
\kappa_j
\;=\;
\frac{1}{|\mathcal{I}_j|}
\sum_{t\in\mathcal{I}_j} \ell_t^{\mathrm{KL},\mathcal{S}},
\label{eq:segment-loss}
\end{equation}
where a larger $\kappa_j$ indicates stronger disagreement with the teacher. Here, token distance refers to the difference $|t-s|$ between two token positions in the response sequence.

\begin{proposition}[Variance Reduction by Segment Averaging]
\label{prop:segment-averaging}
Let $L_j=|\mathcal{I}_j|$, and suppose that within merged segment $j$ the support-restricted token-level losses satisfy
\[
\ell_t^{\mathrm{KL},\mathcal{S}}=\mu_j+\epsilon_t,
\qquad t\in\mathcal{I}_j,
\]
where $\mu_j$ is locally constant, $\mathbb{E}[\epsilon_t]=0$, and
\[
\left|\operatorname{Cov}(\epsilon_t,\epsilon_s)\right|
\le
\sigma^2\rho^{|t-s|}
\qquad
\text{for some }0\le\rho<1.
\]
Then $\kappa_j$ is an unbiased estimator of $\mu_j$, and
\begin{equation}
\operatorname{Var}(\kappa_j)
\le
\frac{\sigma^2}{L_j}\frac{1+\rho}{1-\rho}.
\label{eq:segment-variance-bound}
\end{equation}
Thus, provided that $\sigma^2$ and $\rho$ are bounded independently of $L_j$, its variance decreases as $O(1/L_j)$.
\end{proposition}

A detailed proof is provided in Appendix~\ref{app:proof-segment-averaging}. Proposition~\ref{prop:segment-averaging} formalizes how averaging reduces sensitivity to isolated token-level loss fluctuations when correlations decay with token distance. This merging trades some temporal resolution for a more stable estimate of the teacher--student discrepancy over a coherent progress interval.

We do not assume that proximity to the teacher is universally equivalent to reasoning quality. Instead, we use process rewards to identify local ranking conflicts between the teacher-imitation signal and the estimated contribution of reasoning segments. Specifically, if segment $a$ has a larger estimated process contribution than segment $b$, but is assigned a larger average distillation loss, then the distillation signal penalizes the empirically more useful segment more strongly. We regard such an ordering as locally inconsistent:
\begin{equation}
\widetilde{PR}_a > \widetilde{PR}_b
\quad\text{and}\quad
\kappa_a > \kappa_b.
\label{eq:inconsistent-order}
\end{equation}
Rather than treated as an assumption that teacher proximity directly measures segment quality, this criterion is used as an operational diagnostic. A representative case that violates this expected relationship is illustrated in Figure~\ref{fig:voilation}.

The process reward and segment-level average distillation loss have different scales and need not be calibrated across trajectories. We therefore compare their relative order within each response instead of requiring their absolute values to be directly comparable~\cite{NIPS2017_d5e2c0ad,ouyang2022training}. This within-response ranking is insensitive to trajectory-wide offsets and positive rescaling, and it reduces the influence of absolute-score calibration noise. The ranking determines the expected direction of the relation between progress and average distillation loss, while the original score gaps in Eq.~\eqref{eq:segment-loss} quantify the severity of each detected conflict.

For each response $i$, we exclude the initial prefix segment ($j=1$) and sort the remaining segments by decreasing process-reward progress $\widetilde{\mathrm{PR}}_j$, breaking ties by increasing loss $\kappa_j$. Let $\rho_i$ denote this sorted sequence of segment indices. We form adjacent pairs $\mathcal{P}_i = \{(\rho_{i,r}, \rho_{i,r+1}) \mid \widetilde{\mathrm{PR}}_{\rho_{i,r}} > \widetilde{\mathrm{PR}}_{\rho_{i,r+1}}\}$, where adjacency is defined on the ranked sequence rather than temporal order.

For each pair $(a,b) \in \mathcal{P}_i$, an order violation occurs if the higher-progress segment $a$ incurs a higher distillation loss, quantified as $v_i(a,b) = \left[(\widetilde{\mathrm{PR}}_a - \widetilde{\mathrm{PR}}_b)(\kappa_a - \kappa_b)\right]_{+}$. The inconsistency score of segment $j$ is accumulated across all adjacent pairs it participates in:
\begin{equation}
\mathrm{Inc}_j = \sum_{(a,b) \in \mathcal{P}_i : \, j \in \{a,b\}} v_i(a,b).
\label{eq:segment-inconsistency}
\end{equation}

\paragraph{Segment Masking.}
Masking is applied to eligible responses containing at least $\max(3, n_{\min})$ segments with $|\mathcal{P}_i| > 0$. Given a masking ratio $q\%$, the budget for response $i$ is $b_i = \lceil \frac{q}{100}n_i \rceil$. We rank candidate segments ($j \ge 2$) by $\mathrm{Inc}_j$ in descending order, and select the top-$b_i$ segments with strictly positive inconsistency scores for masking. 

Let $\mathcal{M}_i$ denote the set of selected segment indices in response $i$. The token-level mask is defined as $M_t^{(i)} = 0$ if token $t$ belongs to any segment $j \in \mathcal{M}_i$, and $M_t^{(i)} = 1$ otherwise. Let $Z_i = \sum_{t=1}^{T_i} M_t^{(i)}$ denote the number of unmasked tokens in response $i$. Applying this mask yields the \method{} objective:
\begin{equation}
\mathcal{L}_{\mathrm{R^2\text{-}OPD}}
=
\mathbb{E}_{x_i,\, y_i \sim \pi_S(\cdot \mid x_i)}
\left[
\frac{1}{Z_i}
\sum_{t=1}^{T_i}
M_t^{(i)} \,
\ell_t^{\mathrm{KL},\mathcal{S},(i)}
\right].
\label{eq:r-opd-objective}
\end{equation}
Masked tokens therefore contribute neither to the loss numerator nor to the normalization denominator, keeping the per-response loss scale comparable across different masking ratios. Apart from masking and the corresponding renormalization, the policy-optimization procedure remains otherwise unchanged.
\section{Experiments}
\label{sec:experiments}

\begin{table*}[t]
\centering
\small
\setlength{\tabcolsep}{7pt}
\renewcommand{\arraystretch}{1.15}
\begin{tabular}{l|cc|cc|cc|cc}
\toprule
\multicolumn{1}{c|}{\multirow{2}{*}{Method}}
& \multicolumn{2}{c|}{AIME 24}
& \multicolumn{2}{c|}{AIME 25}
& \multicolumn{2}{c|}{Olympiad}
& \multicolumn{2}{c}{Avg.} \\
& avg@4 & pass@4
& avg@4 & pass@4
& avg@4 & pass@4
& avg@4 & pass@4 \\
\midrule
Student  
& 22.50 & 43.33
& 23.33 & \textbf{36.67}
& 43.19 & 58.31
& 29.67 & 46.10 \\
Teacher 
& 41.67 & 56.67
& 30.83 & 43.44
& 53.28 & 68.43
& 41.92 & 56.18 \\
OPD~\cite{ICLR2024_5be69a58}
& 28.33 & 50.00
& 22.50 & 30.00
& 46.86 & 62.10
& 32.55 & 47.37 \\
E-OPD~\cite{jin2026entropyaware}
& 18.33 & 36.67
& 12.50 & 23.33
& 49.91 & 64.56
& 26.91 & 41.52 \\
TIP-OPD~\cite{xu2026tip}
& 17.50 & 33.33
& 11.67 & 20.00
& 48.24 & 62.83
& 25.80 & 38.72 \\
IW-OPD~\cite{xie2026position}
& 20.83 & 36.67
& 17.50 & 26.67
& 43.59 & 59.09
& 27.31 & 40.81 \\ 
Uni-OPD~\cite{hou2026uni}
& 20.00 & 43.33
& 19.17 & 26.67
& \textbf{53.16} & \textbf{69.97}
& 30.78 & 46.66 \\ 
\textbf{\method{} (Ours)}
& \textbf{32.50} & \textbf{56.67}
& \textbf{25.83} & \textbf{36.67}
& 46.86 & 62.19
& \textbf{35.06} & \textbf{51.83} \\
\bottomrule
\end{tabular}
\caption{OPD performance on DeepSeek-R1-Distill-Qwen-1.5B with JustRL as the teacher model.}
\label{tab:main-results}
\end{table*}

\subsection{Experimental Setup}
\label{sec:experimental-setup}
Our primary experiments employ DeepSeek-R1-Distill-Qwen-1.5B as the student and JustRL-1.5B as the teacher. As highlighted by \citet{li2026rethinking}, effective distillation requires a teacher with complementary knowledge rather than merely a larger parameter scale. The former is a compact model distilled from DeepSeek-R1 and inherits its reasoning-oriented training recipe~\cite{guo2025deepseek}, while the latter already demonstrates that a simple RL recipe can elicit strong reasoning performance from a 1.5B model~\cite{he2025justrl}.  To assess whether our approach generalizes across different model families, we further evaluate on a heterogeneous setup using Qwen3-1.7B~\cite{yang2025qwen3} as the student and e3-1.7B~\cite{setlur2025e3} as the teacher.

All models are trained on the deduplicated DAPO-Math-17K dataset~\cite{li2026rethinking} for one epoch using the AdamW optimizer with a learning rate of $5\times10^{-6}$ and a global batch size of 64. The maximum prompt and response lengths are set to 1,024 and 7,168 tokens, respectively. For the support-restricted reverse-KL objective, we set $H=16$ and construct $\mathcal{S}_t$ from the student's 16 highest-probability tokens at each decoding step. Downstream reasoning performance is evaluated on AIME~2024~\cite{aime_2024}, AIME~2025~\cite{aime_2025}, and OlympiadBench~\cite{he-etal-2024-olympiadbench}. For Rank-Based Conflict Detection and Segment Masking, we generate $N_{\mathrm{eval}}=8$ answer-eliciting rollouts per evaluated boundary using a sampling temperature of 0.7, $\text{top-}k=50$, $\text{top-}p=1.0$, and a maximum rollout length of 300 tokens. Hyperparameters for conflict detection are set to $S_{\min}=3$ and $n_{\min}=3$, with a segment-level masking ratio of $q=30\%$. Complete optimization, rollout-construction, and filtering details are provided in Appendix~\ref{app:training-details}.

\subsection{Evaluation Baselines}
\label{sec:evaluation-baselines}

We compare \method{} with standard OPD and four recent research. \textbf{OPD} applies dense token-level reverse-KL supervision from the teacher to trajectories sampled on-policy from the student~\cite{ICLR2024_5be69a58}. \textbf{E-OPD} augments reverse-KL training with forward KL at positions where the teacher distribution has high entropy, aiming to preserve plausible alternatives and avoid mode collapse~\cite{jin2026entropyaware}. \textbf{TIP-OPD} characterizes token importance using student entropy and the divergence between teacher and student, retaining both uncertain positions and low-entropy positions where the student is confidently misaligned with the teacher~\cite{xu2026tip}. \textbf{Uni-OPD} addresses insufficient student exploration through data balancing and unreliable teacher supervision using outcome-guided margin calibration~\cite{hou2026uni}. \textbf{IW-OPD} studies position bias in OPD and weights tokens according to accumulated student-teacher discrepancy, emphasizing earlier positions while downweighting later positions of less reliable supervision~\cite{xie2026position}.

We report task accuracy using the task-specific answer verifier. For each benchmark, avg@4 averages accuracy over four sampled responses, whereas pass@4 measures whether at least one of the four responses is correct. During evaluation, responses are sampled with a temperature of \(1.0\), top-\(k\) of \(50\), and top-\(p\) of \(0.95\), with the maximum response length set to \(8{,}192\) tokens, which aligns with the literature.

\begin{table}[t]
\centering
\small
\setlength{\tabcolsep}{2pt}
\renewcommand{\arraystretch}{1.10}
\begin{tabular}{l|cc|cc|cc}
\toprule
\multirow{2}{*}{Dataset}
& \multicolumn{2}{c|}{Base}
& \multicolumn{2}{c|}{OPD}
& \multicolumn{2}{c}{\textbf{\method{}}} \\
& avg@4 & pass@4
& avg@4 & pass@4
& avg@4 & pass@4 \\
\midrule
AIME 24
& 24.17 & 30.00
& 22.50 & 36.67
& \textbf{25.00} & \textbf{40.00} \\
AIME 25
& 18.75 & 20.00
& \textbf{27.50} & 33.33
& 25.83 & \textbf{36.67} \\
Olympiad
& 50.10 & 63.87
& 53.82 & 67.09
& \textbf{54.31} & \textbf{67.91} \\
Avg.
& 31.01 & 37.96
& 34.61 & 45.70
& \textbf{35.04} & \textbf{48.19} \\
\bottomrule
\end{tabular}
\caption{Transfer performance of \method{} on Qwen3-1.7B with e3-1.7B as the teacher model.}
\label{tab:qwen3-results}
\end{table}

\subsection{Main Results}
\label{sec:main-results}
We summarize results comparison on DeepSeek-R1-Distill-Qwen-1.5B model in Table~\ref{tab:main-results}. It shows that \method{} achieves the best aggregate performance, with 35.06 avg@4 and 51.83 pass@4. It outperforms standard OPD by 2.51 and 4.46 points, respectively, with the largest gains observed on the two AIME benchmarks. Among the recent OPD variants, Uni-OPD is the strongest overall competitor, yet \method{} exceeds it by 4.28 avg@4 and 5.17 pass@4 points. Uni-OPD remains stronger on OlympiadBench, suggesting the benefit of progress-aware filtering varies with the benchmark.

The relatively weaker AIME results of E-OPD and TIP-OPD reflect an interaction between entropy-sensitive supervision and the response-length constraint. Emphasizing uncertain positions can preserve alternative reasoning, but may also increase the likelihood of lengthy or unfinished reasoning on challenging problems. Uni-OPD may likewise be affected because its correctness-aware calibration relies on complete, verifiable trajectories, whereas our maximum training response length is consistently shorter than that used in its original study, suggesting that \method{} can result in more concise answers for reasoning tasks.

Table~\ref{tab:qwen3-results} further evaluates transfer to the Qwen3-1.7B/e3-1.7B pair. Although the avg@4 gain over OPD is modest, \method{} consistently improves the aggregate pass@4 from 45.70 to 48.19 and obtains higher pass@4 on all three benchmarks. These results suggest that progress-aware filtering transfers across model families, while its effective size may vary across datasets and metrics.

\subsection{Ablations}
We conduct ablation studies to examine two design choices central to \method{}. First, we vary the masking ratio $q$ to assess how the amount of filtered supervision affects downstream performance. Second, we remove sign-consistent segment merging to isolate its contribution to the reliability of PR--KL rank comparisons and the resulting task accuracy. All ablations use the main DeepSeek-R1-Distill-Qwen-1.5B/JustRL-1.5B configuration, with all other training and evaluation settings held fixed unless otherwise specified.

\begin{figure}
    \centering
    \includegraphics[width=0.9\linewidth]{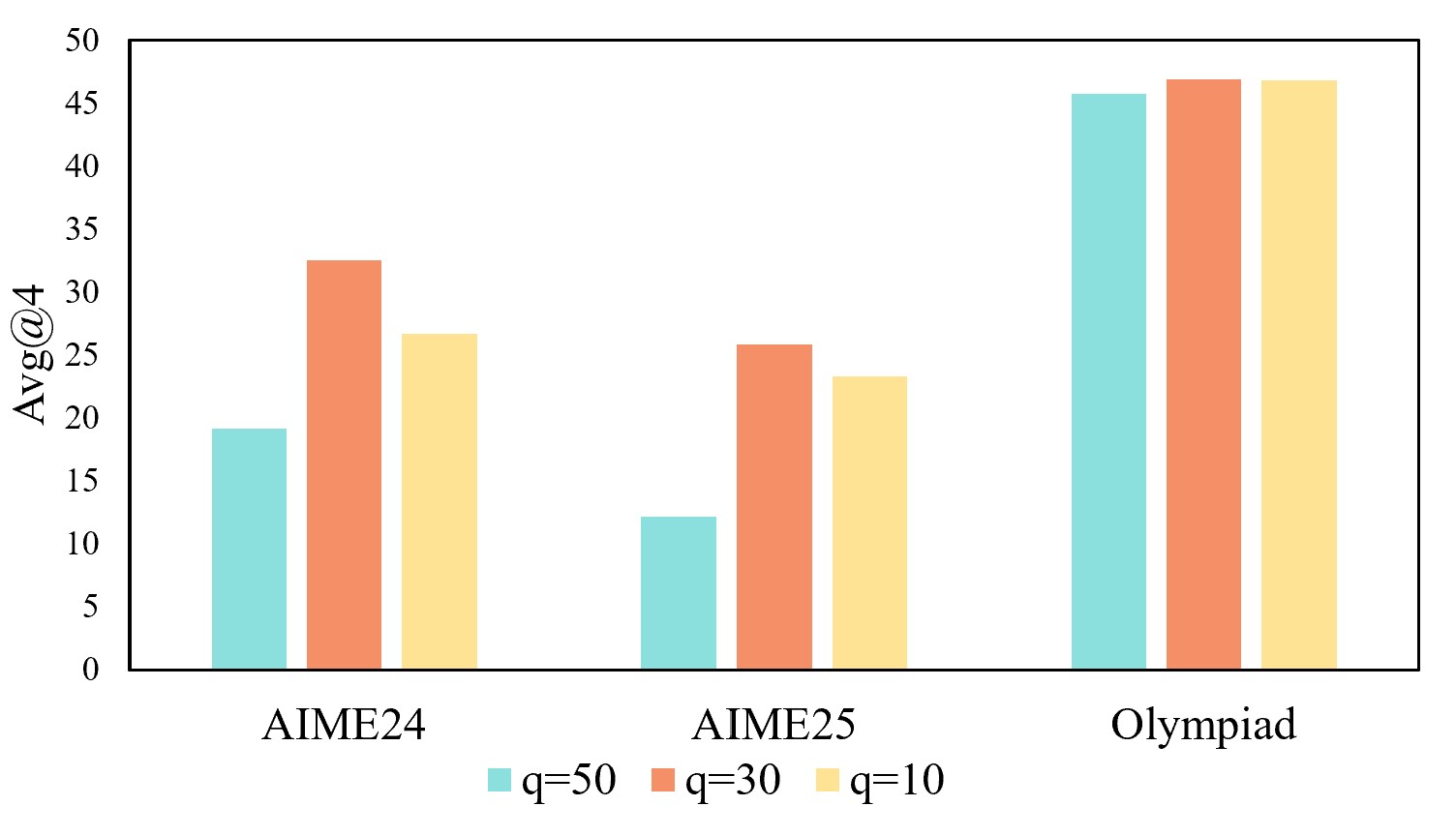}
    \caption{Sensitivity of \method{} to the masking ratio $q$. A moderate ratio of $q=30$ provides the best overall performance, whereas overly aggressive masking substantially degrades accuracy on AIME~2024 and AIME~2025.}
    \label{fig:ablition-q}
\end{figure}

\paragraph{Sensitivity to the Masking Ratio.}
Figure~\ref{fig:ablition-q} examines the effect of the masking ratio $q$, which controls the proportion of merged reasoning segments considered for filtering. A moderate masking ratio of $q=30$ achieves the strongest performance across all three benchmarks, reaching 32.50 on AIME~2024, 25.83 on AIME~2025, and 46.86 on OlympiadBench. Reducing the ratio to $q=10$ yields lower performance on the two AIME benchmarks, suggesting that overly conservative filtering may retain a substantial amount of progress-conflicting supervision. Conversely, increasing $q$ to 50 causes a pronounced degradation, particularly on AIME~2024 and AIME~2025, where accuracy falls to 19.17 and 12.15, respectively. This indicates that overly aggressive masking can remove otherwise useful teacher guidance. Overall, $q=30$ provides a favorable balance between suppressing unreliable distillation signals and preserving informative supervision. Performance on OlympiadBench remains comparatively stable across the three settings, indicating lower sensitivity to the masking ratio on this benchmark.

\paragraph{Effectiveness of Sign-Consistent Merging.}
\begin{figure}
    \centering
    \includegraphics[width=\linewidth]{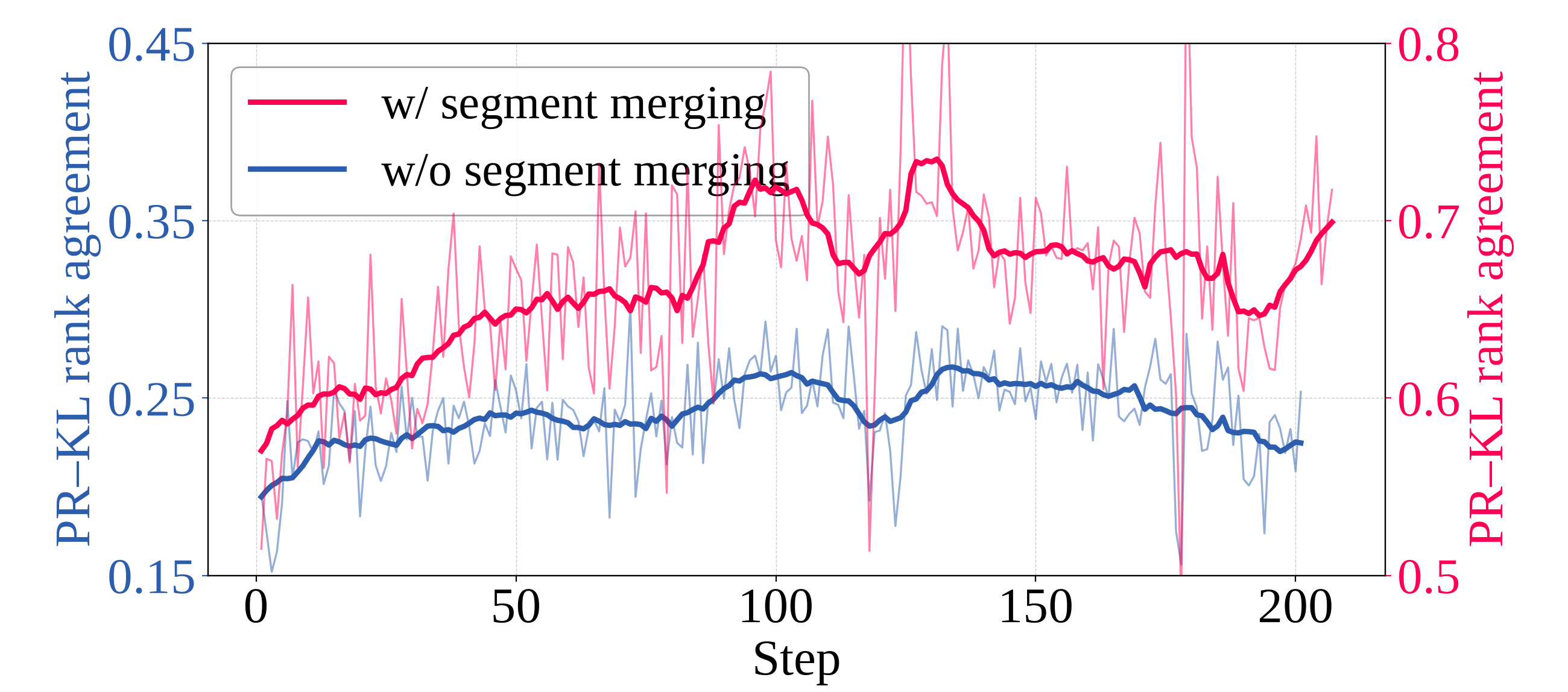}
    \caption{PR--KL rank agreement during training with and without sign-consistent segment merging. Light lines show per-step values, and dark lines show 15-step moving averages. The two vertical axes use the same numerical span. }
    \label{fig:aggregation-consistency}
\end{figure}
To verify if sign-consistent merging improves unit reliability for conflict detection, we measure PR--KL rank agreement at each step. An adjacent segment pair is deemed consistent if higher process reward corresponds to a non-greater distillation loss $\kappa$. We record the average proportion of consistent pairs per response, where 0.5 represents a random baseline and higher values indicate stronger alignment between estimated progress and teacher--student agreement.

As shown in Figure~\ref{fig:aggregation-consistency}, sign-consistent merging raises the smoothed PR--KL rank agreement from approximately $0.20$--$0.27$ to $0.55$--$0.73$ and keeps it above the random baseline for nearly the entire training trajectory. This persistent gap indicates that merging reduces boundary-level noise and yields more reliable units for conflict detection. Consistently, Table~\ref{tab:merging-ablation} shows that merging improves avg@4 by 15.00 and 14.16 points on AIME~2024 and AIME~2025, respectively, while yielding a smaller decrease of 0.96 points on OlympiadBench. Together, these results show that improved rank reliability translates into more effective OPD supervision, particularly on the more challenging AIME benchmarks.

\begin{table}[t]
\centering
\small
\setlength{\tabcolsep}{5pt}
\renewcommand{\arraystretch}{1.10}
\begin{tabular}{l|ccc}
\toprule
Dataset & No Merge & \method{} & $\Delta$ \\
\midrule
AIME 24  & 17.5 & 32.50 & 15.0 \\
AIME 25  & 11.67 & 25.83 &  14.16 \\
Olympiad & 47.8 & 46.86 & -0.96 \\
\bottomrule
\end{tabular}
\caption{Effect of sign-consistent segment merging on avg@4 accuracy. $\Delta$ denotes the improvement of the full \method{} over its no-merging variant.}
\label{tab:merging-ablation}
\end{table}

\section{Related works}
\paragraph{Process Rewards for Reasoning Language Models}
Process rewards evaluate intermediate reasoning states and provide finer-grained credit assignment than terminal outcome rewards \cite{Uesato2022SolvingMW,lightman2024let}.  
To reduce the cost of human annotation, subsequent methods automatically construct process supervision from sampled continuations, outcome verification, or model-based judgments~\cite{wang2024math,Luo2024ImproveMR,setlur2025rewarding,zhang2025lessons,yang-etal-2025-beyond-first}. 
These signals have been used both training-time policy optimization and inference-time candidate selection or search~\cite{zhang2024rest,cui2025process, yuan2026verifiable,xie2026step}. 
However, process rewards can be affected by finite-rollout noise, verifier errors, and distribution shift~\cite{jia2025we,song-etal-2025-prmbench,dontsov-etal-2026-distribution}. Unlike prior work that directly optimizes or searches against process rewards, we use reasoning progress as an independent reliability test for OPD supervision and aggregate adjacent progress-consistent spans to reduce estimation noise.

\paragraph{Reliability and Selectivity of OPD Signals}
Recent work has investigated the conditions under which dense token-level supervision in OPD supports effective capability transfer. Large student–teacher discrepancies can misdirect exploration, while teachers may prefer plausible but incorrect solutions over correct alternative reasoning paths \cite{wang2026demystifying, akhondzadeh2026reward}.  At a coarser granularity, aggregated token-level guidance can rank correct trajectories below incorrect ones, motivating calibration against outcome rewards. Together, these findings suggest that teacher supervision on student rollouts should be assessed rather than applied uniformly~\cite{hou2026uni}.
At the token level, high teacher entropy has motivated selectively combining reverse and forward KL to preserve plausible alternatives, while student entropy and teacher–student divergence have been used to identify informative training positions~\cite{jin2026entropyaware,xu2026tip}. Supervision quality has also been shown to deteriorate at later positions as the discrepancy between student and teacher accumulates along a rollout, motivating position-sensitive importance weighting~\cite{xie2026position}. Existing methods assess OPD signals using token uncertainty, teacher–student discrepancy, sequence position, or trajectory correctness. Our work instead evaluates teacher supervision at the reasoning-span level by examining its consistency with estimated reasoning progress.

\section{Conclusion}
This work identifies a limitation of reasoning-oriented on-policy distillation, where agreement with the teacher distribution does not necessarily reflect whether an intermediate step advances the solution. To this end, we introduce \method{}, which estimates reasoning progress from on-policy continuations, merges adjacent segments with sign-consistent progress, and compares the resulting ranking with that induced by teacher--student divergence. Rather than replacing the OPD objective with an additional reward, \method{} utilizes ranking conflicts as a reliability test and selectively masks supervision that may discourage productive reasoning. Experiments show that this selective treatment of teacher supervision improves aggregate reasoning performance over standard OPD, with particularly strong gains on the more challenging AIME benchmarks and consistent improvements in pass@4. 
Future work will investigate approaches to reduce the cost and variance of process reward estimation and examine whether progress-aware filtering generalizes to larger models and other diverse domains.

\bibliography{aaai2027}

\ifdefined\OpenSourceVersion
\appendix
\setcounter{secnumdepth}{2}
\section{Proofs}
\label{app:detailed-proofs}

\subsection{Proof of Proposition~\ref{prop:segment-averaging}}
\label{app:proof-segment-averaging}

\begin{proof}
Fix a merged segment \(j\) and condition on its contiguous token-index set
\(\mathcal{I}_j=\{u,u+1,\ldots,u+L_j-1\}\).  For notational simplicity,
write \(L=L_j\).  By Eq.~\eqref{eq:segment-loss} and the decomposition
\(\ell_t^{\mathrm{KL},\mathcal{S}}=\mu_j+\epsilon_t\),
\[
\begin{aligned}
\kappa_j
&=
\frac{1}{L}\sum_{t\in\mathcal{I}_j}\ell_t^{\mathrm{KL},\mathcal{S}}\\
&=
\frac{1}{L}\sum_{t\in\mathcal{I}_j}
(\mu_j+\epsilon_t)\\
&=
\mu_j+\frac{1}{L}\sum_{t\in\mathcal{I}_j}\epsilon_t.
\end{aligned}
\]
Taking expectations and using \(\mathbb{E}[\epsilon_t]=0\) for every
\(t\in\mathcal{I}_j\) gives
\[
\mathbb{E}[\kappa_j]
=
\mu_j+\frac{1}{L}\sum_{t\in\mathcal{I}_j}
\mathbb{E}[\epsilon_t]
=
\mu_j.
\]
Therefore, \(\kappa_j\) is an unbiased estimator of the locally constant
segment-average distillation loss \(\mu_j\).

It remains to bound the variance.  Adding the constant \(\mu_j\) does not
change variance, so
\[
\begin{aligned}
\operatorname{Var}(\kappa_j)
&=
\operatorname{Var}\!\left(
\frac{1}{L}\sum_{t\in\mathcal{I}_j}\epsilon_t
\right)\\
&=
\frac{1}{L^2}
\sum_{t\in\mathcal{I}_j}
\sum_{s\in\mathcal{I}_j}
\operatorname{Cov}(\epsilon_t,\epsilon_s).
\end{aligned}
\]
The covariance sum contains \(L\) diagonal terms with \(t=s\).  The
assumed covariance bound implies
\[
\operatorname{Var}(\epsilon_t)
=
\operatorname{Cov}(\epsilon_t,\epsilon_t)
\leq \sigma^2.
\]
For every distance \(d\in\{1,\ldots,L-1\}\), contiguity of
\(\mathcal{I}_j\) implies that exactly \(L-d\) unordered pairs of token
positions satisfy \(|t-s|=d\).  Each unordered pair occurs twice in the
double sum, once as \((t,s)\) and once as \((s,t)\).  Since a covariance
may be negative, we upper-bound each off-diagonal term by its absolute
value.  Hence,
\[
\begin{aligned}
\operatorname{Var}(\kappa_j)
&\leq
\frac{1}{L^2}
\left[
L\sigma^2
+
2\sigma^2\sum_{d=1}^{L-1}(L-d)\rho^d
\right]\\
&=
\frac{\sigma^2}{L^2}
\left[
L+2\sum_{d=1}^{L-1}(L-d)\rho^d
\right].
\end{aligned}
\]
Using \(L-d\leq L\), extending the finite nonnegative sum to an infinite
one, and applying the geometric-series identity
\(\sum_{d=1}^{\infty}\rho^d=\rho/(1-\rho)\), which is valid because
\(0\leq\rho<1\), we obtain
\[
\begin{aligned}
\operatorname{Var}(\kappa_j)
&\leq
\frac{\sigma^2}{L^2}
\left[
L+2L\sum_{d=1}^{\infty}\rho^d
\right]\\
&=
\frac{\sigma^2}{L}
\left(
1+\frac{2\rho}{1-\rho}
\right)\\
&=
\frac{\sigma^2}{L}
\frac{1+\rho}{1-\rho}.
\end{aligned}
\]
Restoring \(L=L_j\) proves Eq.~\eqref{eq:segment-variance-bound}.  If
\(\sigma^2\) and \(\rho\) are bounded independently of \(L_j\), the
multiplicative factor
\(\sigma^2(1+\rho)/(1-\rho)\) is constant with respect to segment length.
Consequently, the variance upper bound scales as \(O(1/L_j)\).
\end{proof}

\section{Training Details}
\label{app:training-details}

This section provides the optimization and sequence settings, training data, on-policy response generation procedure, model-specific reasoning templates, process-reward rollout construction, eligibility and fallback rules, and complete pseudocode of \method{}. Model configurations are described in the main paper and are therefore not repeated here.

\subsection{Optimization and Sequence Settings}
\label{app:training-setup}

Except for its progress-aware masking stage, \method{} follows the standard OPD training loop: the student generates responses on-policy, the teacher supplies token-level distributional supervision on the resulting student contexts, and the student is updated using the reverse-KL objective. Table~\ref{tab:training-configuration} separates the settings of the OPD update from those used only for process-reward estimation and segment filtering.

\begin{table}[t]
\centering
\small
\setlength{\tabcolsep}{6pt}
\renewcommand{\arraystretch}{1.18}
\begin{tabular}{@{}p{0.55\columnwidth}p{0.36\columnwidth}@{}}
\toprule
Setting & Value \\
\midrule
\multicolumn{2}{@{}l}{\textit{Optimization and sequence settings}} \\
Training epochs & 1 \\
Optimizer & AdamW \\
Learning rate & $5\times10^{-6}$ \\
Global batch size & 64 \\
Maximum prompt length & 1,024 tokens \\
Maximum response length & 7,168 tokens \\
KL support size & Student top-16 ($H=16$) \\
\addlinespace[2pt]
\multicolumn{2}{@{}l}{\textit{Process-reward estimation}} \\
Rollouts per evaluated boundary & $N_{\mathrm{eval}}=8$ \\
Sampling temperature & 0.7 \\
Top-$k$ & 50 \\
Top-$p$ & 1.0 \\
Maximum rollout length & 300 tokens \\
\addlinespace[2pt]
\multicolumn{2}{@{}l}{\textit{Segment filtering}} \\
Minimum sentences between boundaries & $S_{\min}=3$ \\
Minimum merged segments & $n_{\min}=3$ \\
Segment masking ratio & $q=30\%$ \\
\bottomrule
\end{tabular}
\caption{Optimization, process-reward estimation, and filtering settings used in all \method{} experiments.}
\label{tab:training-configuration}
\end{table}

\subsection{Training Data}
\label{app:training-data}

All models are trained for one epoch on the same deduplicated DAPO-Math-17K corpus used in the main experiments. All reported configurations use the same prompt set, so their differences do not arise from configuration-specific data selection. Each training example consists of a mathematical prompt $x_i$ and a verifier-compatible ground-truth answer $g_i$. Responses are generated by the current student during training rather than precomputed or replaced by teacher-generated solutions.

\subsection{On-Policy Response Generation}
\label{app:on-policy-generation}

For each prompt $x_i$, the current student samples an on-policy response $y_i$ with a maximum length of 7,168 tokens. This response is the training trajectory on which the OPD loss is defined. The teacher is queried on the same token contexts to compute the support-restricted reverse-KL loss. At each position, the approximate support contains the 16 tokens assigned the highest probability by the student, as specified by $H=16$.

The token mask is initialized to one over the entire response. \method{} subsequently changes this mask only when the response satisfies all process-reward and conflict-detection requirements described below. The short answer-eliciting rollouts used for boundary evaluation are separate generations: they estimate intermediate solve probabilities, but their tokens are not included in the KL objective and are not used as replacement training trajectories.

\subsection{Model-Specific Reasoning Templates}
\label{app:reasoning-templates}

We use the chat template shipped with each model checkpoint rather than imposing a single wrapper across model families. The student and teacher in each configuration share the same family-specific format: DeepSeek-R1-Distill-Qwen-1.5B and JustRL-DeepSeek-1.5B use the DeepSeek reasoning template, whereas Qwen3-1.7B and e3-1.7B use the Qwen3 thinking-mode template. This distinction is important because the former is based on Qwen2.5 but does not use the standard Qwen2.5 ChatML wrapper.

\lstdefinestyle{templatebox}{
  basicstyle=\footnotesize\ttfamily,
  numbers=none,
  frame=single,
  framerule=0.5pt,
  framesep=3pt,
  framextopmargin=2pt,
  framexbottommargin=2pt,
  framexleftmargin=0pt,
  framexrightmargin=0pt,
  linewidth=\columnwidth,
  breaklines=true,
  breakatwhitespace=false,
  columns=fullflexible,
  keepspaces=true,
  showstringspaces=false,
  aboveskip=5pt,
  belowskip=5pt,
  xleftmargin=0pt,
  xrightmargin=0pt
}

Across the two configurations, the benchmark problem is inserted into the following common task-level instruction:
\begin{lstlisting}[style=templatebox]
Solve the following math problem step by step.
Put your answer inside \boxed{}.

{problem}

Remember to put your answer inside \boxed{}.
\end{lstlisting}
Only the model-specific special-token wrapper differs. The schematics below give the single-turn path used in our experiments and preserve the ordering and line breaks of the corresponding tokenizer templates. For the DeepSeek template, we use ASCII aliases for its Unicode special tokens to ensure compatibility with the paper's \LaTeX{} engine: \texttt{<DS-BOS>}, \texttt{<DS-User>}, \texttt{<DS-Assistant>}, and \texttt{<DS-EOS>} denote the checkpoint's beginning-of-sequence, user, assistant, and end-of-sequence tokens, respectively. These aliases are used only for typesetting; the implementation uses the literal tokens from the tokenizer file.

\paragraph{Qwen2.5-based DeepSeek group.}
For both the DeepSeek-R1-Distill-Qwen-1.5B student and the JustRL-DeepSeek-1.5B teacher, applying the checkpoint template with \texttt{add\_generation\_prompt=True} produces a prefix ending in the assistant marker followed immediately by an opening \texttt{<think>} tag:
\begin{lstlisting}[style=templatebox]
<DS-BOS><DS-User>{task prompt}<DS-Assistant><think>
{reasoning}
</think>

{final answer}<DS-EOS>
\end{lstlisting}
Thus, the opening reasoning tag is part of the formatted generation prefix. The generated continuation supplies the reasoning content, closes the tag, emits the final answer, and terminates with the checkpoint's end-of-sequence token. We retain this DeepSeek formatting throughout response generation and teacher scoring; we do not replace it with Qwen2.5's \texttt{<|im\_start|>} and \texttt{<|im\_end|>} markers.

\paragraph{Qwen3 group.}
For the Qwen3-1.7B student and e3-1.7B teacher, we apply the Qwen3 template in thinking mode, i.e., with \texttt{enable\_thinking=True}. The corresponding single-turn sequence is:
\begin{lstlisting}[style=templatebox]
<|im_start|>user
{task prompt}<|im_end|>
<|im_start|>assistant
<think>
{reasoning}
</think>

{final answer}<|im_end|>
\end{lstlisting}
Unlike the DeepSeek template, the Qwen3 generation prefix ends after the assistant-role newline; in thinking mode, the model itself generates the \texttt{<think>} block before the final response. Setting \texttt{enable\_thinking=False} would instead insert an empty \texttt{<think>}\texttt{</think>} block into the prompt, but this non-thinking path is not used in our reasoning experiments.

These templates format the original on-policy responses used by the OPD objective. They are distinct from the additional answer-eliciting instruction \(\mathcal{I}_{\mathrm{ans}}\) used only to construct boundary-evaluation rollouts in Appendix~\ref{app:answer-instruction}.

\subsection{Process-Reward Rollout Construction}
\label{app:rollout-construction}

For an eligible response, process rewards are estimated by evaluating whether the student can still reach the ground-truth answer from selected intermediate reasoning states. We call the generations used for this evaluation \emph{answer-eliciting rollouts}; after this definition, we refer to them simply as rollouts. Each rollout is sampled from the current student after conditioning on a standardized reasoning prefix and a fixed answer-eliciting instruction. The construction of the evaluated boundaries and standardized prompts is detailed next.

\subsubsection{Segmentation Lexicon and Boundary Matching}
\label{app:segmentation}

Table~\ref{tab:segmentation-lexicon} lists the complete discourse-marker
lexicon used by the segmentation procedure. The categories are descriptive;
matching is performed only against the literal expressions shown in the
table.

Before matching, the decoded response is treated as plain text. Matching is
case-insensitive, uses word boundaries, and permits one or more whitespace
characters between the words of a multiword marker. Matches are processed
from left to right. A candidate cut is placed immediately before the matched
marker, so the marker begins the next segment. The cut is retained only when
the text since the previous accepted boundary contains at least \(S_{\min}\)
sentences, where the implementation approximates the sentence count by the
number of periods. If no cut is accepted, the complete response
is treated as a single segment.

\subsubsection{Answer-Eliciting Instruction and Prefix Standardization}
\label{app:answer-instruction}

The symbol \(\mathcal{I}_{\mathrm{ans}}\) denotes a fixed textual suffix,
rather than a learned component. It instructs the student to stop extending
the reasoning chain and emit a verifier-compatible final answer based on the
reasoning already contained in the prefix. Its exact wording is:

\begin{quote}
\small\ttfamily
Based on the reasoning above, directly give the final answer. Put the final
answer within \textbackslash{}boxed\{\}. If no definitive answer can be derived
from the existing reasoning, output \textbackslash{}boxed\{no answer\}.
\end{quote}

A segment boundary commonly falls inside an open \texttt{<think>} block. If
the truncated response prefix contains more occurrences of
\texttt{<think>} than \texttt{</think>}, we remove trailing newline
characters and append \texttt{</think>} followed by a blank line before
\(\mathcal{I}_{\mathrm{ans}}\). Otherwise, the prefix is left unchanged.
Closing the unmatched tag supplies the mode-transition cue needed for the
answer-eliciting rollout to follow the answer instruction instead of treating
it as additional reasoning. This preprocessing is applied only to the
rollout prompt; it does not alter the original student trajectory
used for training.

\subsubsection{Boundary Rollouts and Solve-Probability Estimation}
\label{app:boundary-rollouts}

Let $p_m=(x,y_{1:b_m})$ denote the state at the end of segment $m$. For every nonterminal boundary of a PR-available response, we draw $N_{\mathrm{eval}}=8$ rollouts from the student conditioned on the standardized prefix $\bar p_m$ and \(\mathcal{I}_{\mathrm{ans}}\). Rollouts use temperature 0.7, top-$k=50$, top-$p=1.0$, and a maximum generation length of 300 tokens. The task-specific verifier then evaluates whether the completed sequence reaches the ground-truth answer $g_i$.

The intermediate solve probability $\hat S_m$ is the empirical fraction of successful rollouts. Following Eq.~\eqref{eq:solve-prob}, we set $\hat S_0=0$ and reuse the verifier result of the original on-policy response for the terminal state $\hat S_M$. Consequently, additional rollouts are required only for the $M-1$ nonterminal boundaries. The process reward of segment $m$ is the finite difference $PR_m=\hat S_m-\hat S_{m-1}$.

\subsection{Eligibility and Fallback Rules}
\label{app:eligibility-rules}

The implementation uses conservative fallback behavior: whenever a response does not provide enough information to define a reliable ranking conflict, its original OPD supervision is retained. Table~\ref{tab:fallback-rules} lists the complete decision logic.

\begin{table}[t]
\centering
\small
\setlength{\tabcolsep}{5pt}
\renewcommand{\arraystretch}{1.18}
\begin{tabular}{@{}p{0.46\columnwidth}p{0.48\columnwidth}@{}}
\toprule
Condition & Action \\
\midrule
The string-level pre-check does not find $g_i$ in $y_i$, or segmentation yields fewer than two segments.
& Skip process-reward rollouts and retain all token-level OPD supervision. \\
The response has fewer than $\max(3,n_{\min})$ merged segments.
& Do not apply segment masking. \\
No strict adjacent pair is available after process-reward ranking.
& Do not apply segment masking. \\
All candidate segments have zero inconsistency score.
& Do not apply segment masking. \\
At least one candidate has a positive inconsistency score.
& Mask up to the response-specific budget; retain every other token. \\
\bottomrule
\end{tabular}
\caption{Eligibility checks and conservative fallback actions used by \method{}.}
\label{tab:fallback-rules}
\end{table}

\begin{table}[t]
\centering
\small
\begin{tabular}{@{}p{0.4\columnwidth}p{0.6\columnwidth}@{}}
\toprule
Category & Matched markers \\
\midrule
Reconsideration &
\texttt{wait}, \texttt{hold on}, \texttt{let me reconsider}, \texttt{hmm} \\
Correction &
\texttt{actually} \\
Verification &
\texttt{let me check} \\
Alternative reasoning &
\texttt{alternatively} \\
\bottomrule
\end{tabular}
\caption{Complete lexicon of discourse markers used to propose reasoning-segment boundaries.}
\label{tab:segmentation-lexicon}
\end{table}

\subsection{Sign-Consistent Filtering and Mask Construction}
\label{app:mask-construction}

The segment-level finite differences are first merged into maximal runs with the same sign, with zero-valued rewards absorbed into the current run. For each merged segment $j$, we sum its process rewards to obtain $\widetilde{PR}_{i,j}$ and average its token-level KL losses to obtain $\kappa_{i,j}$. Candidate segments exclude the first merged segment, which is retained to avoid removing the initial problem interpretation and setup.

The remaining segments are ordered by decreasing $\widetilde{PR}_{i,j}$, with increasing $\kappa_{i,j}$ used to break equal-reward ties. Only adjacent pairs with a strict process-reward difference are compared. A positive inconsistency contribution occurs when the segment with greater estimated progress also has greater KL loss, indicating that teacher--student disagreement is larger on the segment that contributes more to solving the problem. Contributions incident to each segment are accumulated into $\mathrm{Inc}_{i,j}$.

For response $i$, the masking budget is $b_i=\lceil(q/100)n_i\rceil$. Among segments $j\geq2$ with $\mathrm{Inc}_{i,j}>0$, we mask at most the $b_i$ highest-scoring segments. If fewer than $b_i$ candidates have positive scores, all positive-score candidates are masked and the unused budget is not reassigned. All unselected segments retain their original supervision. The final loss is normalized by the number of unmasked tokens, as in Eq.~\eqref{eq:r-opd-objective}, so responses with different retained lengths remain comparable within the batch.

\subsection{Training Procedure}
\label{app:training-algorithm}

Algorithm~\ref{alg:r-opd} summarizes the complete batch-level procedure. The student first generates on-policy responses and receives token-level reverse-KL supervision from the teacher. Eligible responses are then evaluated using answer-eliciting rollouts, merged into sign-consistent progress units, checked for conflicts with segment-average KL losses, and selectively masked before the normalized OPD update.

\begin{algorithm}[t]
\caption{Training procedure of \method{} for a single batch.}
\label{alg:r-opd}
\small
\begin{algorithmic}[1]
\Require Batch $\mathcal{B}=\{(x_i,g_i)\}$, student $\pi_S$, teacher $\pi_T$
\Require $H,N_{\mathrm{eval}},q,S_{\min},n_{\min},M_{\max}$
\Ensure Updated student $\pi_S$
\For{each $(x_i,g_i)\in\mathcal{B}$}
    \State Sample $y_i\sim\pi_S(\cdot\mid x_i)$ and set $M_t^{(i)}\gets 1$ for all $t$
    \State Construct each student top-$H$ support $\mathcal{S}_t^{(i)}$ and compute $\ell_t^{\mathrm{KL},\mathcal{S},(i)}$ \Comment{Eq.~\eqref{eq:support-loss}}
    \State $\Sigma_i\gets\textsc{Segment}(y_i;S_{\min},M_{\max})$
    \If{$g_i\notin y_i$ \textbf{or} $|\Sigma_i|<2$}
        \State \textbf{continue} \Comment{PR-unavailable; retain all supervision}
    \EndIf
    \State $M_i\gets|\Sigma_i|$, $\hat S_{i,0}\gets0$, and $\hat S_{i,M_i}\gets\mathcal{R}(y_i,g_i)$
    \For{$m=1$ \textbf{to} $M_i-1$}
        \State Sample $\{c_{i,m}^{(\ell)}\}_{\ell=1}^{N_{\mathrm{eval}}}\sim\pi_S(\cdot\mid\bar p_{i,m},\mathcal{I}_{\mathrm{ans}})$
        \State $\hat S_{i,m}\gets\frac{1}{N_{\mathrm{eval}}}\sum_{\ell=1}^{N_{\mathrm{eval}}}\mathcal{R}(p_{i,m}\mathbin{\Vert}c_{i,m}^{(\ell)},g_i)$
    \EndFor
    \State $PR_{i,m}\gets\hat S_{i,m}-\hat S_{i,m-1}$ for $m=1,\ldots,M_i$
    \State $(\{\mathcal{J}_{i,j}\},\{\mathcal{I}_{i,j}\})_{j=1}^{n_i}\gets\textsc{MergeSameSignRuns}(\Sigma_i,PR_i)$ \Comment{Zeros are absorbed}
    \For{$j=1$ \textbf{to} $n_i$}
        \State $\widetilde{PR}_{i,j}\gets\sum_{m\in\mathcal{J}_{i,j}}PR_{i,m}$ \Comment{Eq.~\eqref{eq:merged-reward}}
        \State $\kappa_{i,j}\gets|\mathcal{I}_{i,j}|^{-1}\sum_{t\in\mathcal{I}_{i,j}}\ell_t^{\mathrm{KL},\mathcal{S},(i)}$ \Comment{Eq.~\eqref{eq:segment-loss}}
    \EndFor
    \If{$n_i<\max(3,n_{\min})$}
        \State \textbf{continue} \Comment{Response is ineligible for masking}
    \EndIf
    \State $\rho_i\gets\textsc{Sort}(\{2,\ldots,n_i\})$ by decreasing $\widetilde{PR}_{i,j}$, then increasing $\kappa_{i,j}$
    \State $\mathcal{P}_i\gets\{(\rho_{i,r},\rho_{i,r+1}):\widetilde{PR}_{i,\rho_{i,r}}>\widetilde{PR}_{i,\rho_{i,r+1}}\}$
    \If{$|\mathcal{P}_i|=0$}
        \State \textbf{continue}
    \EndIf
    \State $\mathrm{Inc}_{i,j}\gets0$ for $j=1,\ldots,n_i$
    \For{each $(a,b)\in\mathcal{P}_i$}
        \State $v_i(a,b)\gets[(\widetilde{PR}_{i,a}-\widetilde{PR}_{i,b})(\kappa_{i,a}-\kappa_{i,b})]_+$
        \State $\mathrm{Inc}_{i,a}\gets\mathrm{Inc}_{i,a}+v_i(a,b)$; $\mathrm{Inc}_{i,b}\gets\mathrm{Inc}_{i,b}+v_i(a,b)$
    \EndFor
    \State $b_i\gets\lceil(q/100)n_i\rceil$
    \State $\mathcal{M}_i\gets$ top-$b_i$ indices in $\{j\ge2:\mathrm{Inc}_{i,j}>0\}$ by decreasing $\mathrm{Inc}_{i,j}$
    \State Set $M_t^{(i)}\gets0$ iff $t\in\mathcal{I}_{i,j}$ for some $j\in\mathcal{M}_i$
\EndFor
\State $Z_i\gets\sum_{t=1}^{T_i}M_t^{(i)}$ for each $i$
\State $\widehat{\mathcal{L}}_{\mathrm{R^2\text{-}OPD}}\gets\frac{1}{|\mathcal{B}|}\sum_i\frac{1}{Z_i}\sum_{t=1}^{T_i}M_t^{(i)}\ell_t^{\mathrm{KL},\mathcal{S},(i)}$
\State Update $\pi_S$ by minimizing $\widehat{\mathcal{L}}_{\mathrm{R^2\text{-}OPD}}$ \Comment{Eq.~\eqref{eq:r-opd-objective}}
\end{algorithmic}
\end{algorithm}

All mask entries are initialized to one. Therefore, every early \textbf{continue} in Algorithm~\ref{alg:r-opd} implements the conservative fallback described in Table~\ref{tab:fallback-rules}.

\onecolumn
\section{Qualitative Case Studies}
\label{app:case-studies}

We provide two AIME~2024 examples in which \method{} produces the verified answer. This section uses the full width of the two-column page so that the complete response boxes remain readable. Unlike a two-column float, each box is part of the normal document flow and can therefore break automatically across pages while respecting the page margins. Each box reproduces the complete evaluator-facing response recorded for the selected sample; blank lines are suppressed uniformly for compact typesetting, and the two non-ASCII mathematical symbols are transliterated as \texttt{omega} and \texttt{sqrt} to ensure robust \LaTeX{} compilation. Responses that reach the generation limit are reproduced through their recorded endpoint and marked as truncated.

\paragraph{Model setting.}
Both cases use the Qwen2.5-based configuration from the main experiments. The base-model responses are generated by the original DeepSeek-R1-Distill-Qwen-1.5B checkpoint. For standard OPD and \method{}, the student is initialized from this same checkpoint and distilled from the same JustRL-DeepSeek-1.5B teacher on DAPO-Math-17K. The two distilled students therefore share the model initialization, teacher, training data, and evaluation protocol; their difference is that standard OPD retains token-level reverse-KL supervision throughout the response, whereas \method{} filters supervision on progress-conflicting segments. All displayed responses use the same AIME~2024 prompt format and an 8,192-token evaluation limit.

Within each incorrect response, the bold label and underlined passages identify the erroneous transition together with its immediate consequence. Longer annotations are split across adjacent lines so that standard underlining remains within the page boundary. We use no color or background shading.

\lstdefinestyle{caseresponsewide}{
  basicstyle=\scriptsize\ttfamily,
  numbers=none,
  frame=single,
  framerule=0.5pt,
  framesep=3pt,
  framextopmargin=3pt,
  framexbottommargin=4pt,
  framexleftmargin=0pt,
  framexrightmargin=0pt,
  linewidth=\textwidth,
  breaklines=true,
  breakatwhitespace=false,
  columns=fullflexible,
  keepspaces=true,
  showstringspaces=false,
  emptylines=0,
  escapeinside={(*@}{@*)},
  aboveskip=6pt,
  belowskip=8pt,
  xleftmargin=0pt,
  xrightmargin=0pt
}

\subsection{Case 1: Hyperbola-Constrained Rhombus}
\label{app:case-study-hyperbola}

\noindent\fbox{\parbox{\dimexpr\textwidth-2\fboxsep-2\fboxrule\relax}{
\small\textbf{Problem.}
Let \(A\), \(B\), \(C\), and \(D\) be points on the hyperbola
\(\frac{x^2}{20}-\frac{y^2}{24}=1\) such that \(ABCD\) is a rhombus whose diagonals intersect at the origin. Find the greatest real number that is less than \(BD^2\) for all such rhombi.
\hfill\textbf{Reference answer:} \(\boxed{480}\).
}}

\medskip
\noindent\textbf{Base-model response (incorrect: \(\boxed{80}\)).}
\begin{lstlisting}[style=caseresponsewide]
Given the hyperbola \(\frac{x^2}{20} - \frac{y^2}{24} = 1\), we need to find the greatest real number that is less than \(BD^2\) for all rhombi \(ABCD\) whose diagonals intersect at the origin. 

1. The hyperbola has vertices at \((\pm \sqrt{20}, 0)\) and \((0, \pm \sqrt{24})\).
2. Points \(A\), \(B\), \(C\), and \(D\) form a rhombus with diagonals intersecting at the origin, implying \(A\) and \(B\) are symmetric with respect to the origin.
3. The slopes of the diagonals \(AC\) and \(BD\) are negative reciprocals, leading to the condition \(\frac{y_1}{x_1} \times \frac{y_2}{x_2} = -1\).
4. Expressing \(BD^2\) in terms of coordinates of \(B\) and \(D\), we get \(BD^2 = 4(x_2^2 + y_2^2)\).
5. Using the hyperbola equation and the slope condition, we express \(x_2^2\) and \(y_2^2\) in terms of a parameter \(k = \frac{x_1}{y_1}\).
6. Substituting and simplifying, we find \(BD^2\) as a function of \(k^2\), denoted \(t\).
7. The function \(f(t) = \frac{960 (1 + t)}{12 - 10 t}\) is analyzed, showing it is increasing for \(0 \leq t < \frac{6}{5}\).
(*@\textbf{[ERRONEOUS STEP]} \underline{\texttt{8. The minimum value of } \(BD^2\) \texttt{ occurs at } \(t=0\)\texttt{, giving } \(BD^2=80\)\texttt{.}}@*)

Thus, the greatest real number that is less than \(BD^2\) for all such rhombi is \(\boxed{80}\).
\end{lstlisting}

\noindent\textbf{OPD response (incorrect: \(\boxed{80}\)).}
\begin{lstlisting}[style=caseresponsewide]
Given the hyperbola \(\frac{x^2}{20} - \frac{y^2}{24} = 1\), points \(A\), \(B\), \(C\), and \(D\) lie on this hyperbola such that \(ABCD\) is a rhombus with diagonals intersecting at the origin. We need to find the greatest real number less than \(BD^2\) for all such rhombi.

1. The hyperbola equation is \(\frac{x^2}{20} - \frac{y^2}{24} = 1\).
2. The rhombus diagonals intersect at the origin and are perpendicular, implying their slopes multiply to -1.
3. Points \(A\), \(B\), \(C\), and \(D\) lie on the hyperbola and are symmetric with respect to the origin.
4. The coordinates of points \(A\), \(B\), \(C\), and \(D\) satisfy the hyperbola equation and the perpendicularity condition.
5. The condition \(abcd + bcd = 0\) and the hyperbola constraints lead to the expression for \(BD^2\).
6. Using the conditions and parametrization, we find \(BD^2 = 4(c^2 + d^2)\).
7. Expressing \(BD^2\) in terms of \(a\) and \(b\), we find
(*@\textbf{[ERRONEOUS STEP]} \underline{\(BD^2=80+\frac{4400}{11t-720}\), where \(t=a^2\).}@*)
(*@\underline{\texttt{8. As } \(t\) \texttt{ approaches infinity, } \(BD^2\) \texttt{ approaches 80 from above.}}@*)
9. The greatest real number less than \(BD^2\) for all such rhombi is the infimum, which is 80.

Thus, the answer is \(\boxed{80}\).
\end{lstlisting}

\noindent\textbf{\method{} response (correct: \(\boxed{480}\)).}
\begin{lstlisting}[style=caseresponsewide]
Given the hyperbola \(\frac{x^2}{20} - \frac{y^2}{24} = 1\), we need to find the greatest real number that is less than \(BD^2\) for all rhombi \(ABCD\) with diagonals intersecting at the origin.

1. The vertices of the rhombus are \(A, B, C, D\) on the hyperbola, with diagonals intersecting at the origin.
2. The coordinates of the vertices can be represented as \(A(a, b)\), \(B(c, d)\), \(C(-a, -b)\), and \(D(-c, -d)\).
3. The condition for the diagonals to be perpendicular is \(ac + bd = 0\).
4. The distances \(AB\) and \(BC\) must be equal, leading to the equation \(ac + bd = 0\).
5. Using the hyperbola equation, we express \(b^2\) and \(d^2\) in terms of \(a\) and \(c\).
6. The distance \(BD^2\) is calculated as \(4(c^2 + d^2)\).
7. By substituting and simplifying, we find that \(BD^2 = 480 + \frac{288000}{u}\) where \(u = 11t - 720\).
8. As \(u\) approaches infinity, \(BD^2\) approaches 480 from above, making the infimum of \(BD^2\) equal to 480.

Thus, the greatest real number that is less than \(BD^2\) for all such rhombi is \(\boxed{480}\).
\end{lstlisting}

The base model and OPD both admit the boundary value \(t=0\) without establishing that it corresponds to a valid rhombus on the hyperbola, yielding 80. In contrast, \method{} preserves the geometric constraints through the limiting argument and correctly identifies 480 as the infimum of \(BD^2\).

\subsection{Case 2: Product over Roots of Unity}
\label{app:case-study-roots}

\noindent\fbox{\parbox{\dimexpr\textwidth-2\fboxsep-2\fboxrule\relax}{
\small\textbf{Problem.}
Let \(\omega\neq1\) be a 13th root of unity. Find the remainder when
\(\prod_{k=0}^{12}(2-2\omega^k+\omega^{2k})\) is divided by 1000.
\hfill\textbf{Reference answer:} \(\boxed{321}\).
}}

\medskip
\noindent\textbf{Base-model response (incorrect: \(\boxed{191}\)).}
\begin{lstlisting}[style=caseresponsewide]
Given \(\omega \neq 1\) as a primitive 13th root of unity, we need to find the remainder when \(\prod_{k=0}^{12}(2 - 2\omega^k + \omega^{2k})\) is divided by 1000.

First, we recognize that \(\omega\) satisfies \(\omega^{13} = 1\) and the minimal polynomial for \(\omega\) is the 13th cyclotomic polynomial, which is \(\Phi_{13}(x) = x^{12} + x^{11} + \dots + x + 1\).

We need to evaluate the product \(\prod_{k=0}^{12}(2 - 2\omega^k + \omega^{2k})\). By substituting \(j = 2k \mod 13\), we note that since 2 is coprime to 13, the substitution is bijective.
(*@\textbf{[ERRONEOUS STEP]} \underline{\texttt{Thus, the product simplifies to } \(\prod_{j=0}^{12}(2-\omega^j)\)\texttt{.}}@*)

We recognize that this product is equivalent to evaluating the polynomial \(P(x) = 2 - 2x + x^2\) at \(x = \omega^j\) for all \(j\).
(*@\underline{\texttt{It then concludes that } \(\prod_{j=0}^{12}(2-\omega^j)=2^{13}-1\)\texttt{.}}@*)

Calculating this, we get:
\[2^{13} - 1 = 8192 - 1 = 8191\]

Finally, we find the remainder when 8191 is divided by 1000:
\[8191 \mod 1000 = 191\]

Thus, the remainder is \(\boxed{191}\).
\end{lstlisting}

\noindent\textbf{OPD response (incorrect; truncated at 8,192 tokens, extracted answer \(1\)).}
\begin{lstlisting}[style=caseresponsewide]
<think>
Okay, so I've got this problem here where I need to find the remainder when the product \(\prod_{k=0}^{12}(2 - 2\omega^k + \omega^{2k})\) is divided by 1000, where \(\omega \neq 1\) is a 13th root of unity. Hmm, let's try to break this down step by step.

First, let's recall that \(\omega\) is a primitive 13th root of unity, meaning \(\omega^{13} = 1\) and \(\omega^k \neq 1\) for \(k = 1, 2, ..., 12\). So, all the powers of \(\omega\) cycle every 13 terms.

The expression inside the product is \(2 - 2\omega^k + \omega^{2k}\). Maybe I can rewrite this expression in a more manageable form. Let me factor or simplify it.

Let me note that \(\omega^{2k}\) is the same as \((\omega^k)^2\). So, the term becomes \(2 - 2\omega^k + (\omega^k)^2\). Wait, that looks familiar. If I let \(x = \omega^k\), then the expression becomes \(2 - 2x + x^2\). Which is a quadratic in x.

So, \(x^2 - 2x + 2\). Hmm, maybe that can be factored or expressed in terms of roots of unity.

Alternatively, perhaps it's a quadratic in x, so maybe each term is a quadratic polynomial evaluated at \(x = \omega^k\). Since \(\omega^k\) are the 13th roots of unity, perhaps we can relate the product over all k to some polynomial evaluated at specific points.

I remember that the product over roots of unity can sometimes be connected to cyclotomic polynomials or something similar. Maybe I can express this product as a product over k of a quadratic in \(\omega^k\), and then relate that to the value of a polynomial at specific points.

Let me consider the product \(\prod_{k=0}^{12} (2 - 2\omega^k + \omega^{2k})\). Since \(\omega\) is a primitive 13th root, the minimal polynomial is the 13th cyclotomic polynomial, which is \(\Phi_{13}(x) = x^{12} + x^{11} + \dots + x + 1\), but I'm not sure if that's directly helpful here.

Alternatively, maybe consider that each term \(2 - 2\omega^k + \omega^{2k}\) can be rewritten as \( (\sqrt{2})^2 - 2\omega^k + (\omega^k)^2 \). Wait, that resembles the expression for a quadratic in \(\omega^k\). Alternatively, perhaps factor it as \((1 + \omega^k)^2 + 1\), but let me check.

Wait, let's compute \(2 - 2\omega^k + \omega^{2k}\).

Let me denote \(z = \omega^k\), then the term becomes \(2 - 2z + z^2 = z^2 - 2z + 2\). Hmm, maybe we can factor this quadratic?

The discriminant is \(4 - 8 = -4\), so the roots are \( [2 \pm \sqrt{-4}]/2 = 1 \pm i\). Hmm, so \(z^2 - 2z + 2 = (z - (1 + i))(z - (1 - i))\). So, \(2 - 2z + z^2 = (z - (1 + i))(z - (1 - i))\).

Therefore, substituting back \(z = \omega^k\), we get:

\((\omega^k - (1 + i))(\omega^k - (1 - i))\).

Therefore, the product becomes:

\(\prod_{k=0}^{12} (\omega^k - (1 + i))(\omega^k - (1 - i)) = \left( \prod_{k=0}^{12} (\omega^k - (1 + i)) \right) \left( \prod_{k=0}^{12} (\omega^k - (1 - i)) \right)\).

Now, that seems promising because the product over (z - \omega^k) is related to the cyclotomic polynomial evaluated at z.

In general, for a primitive nth root of unity, the product \(\prod_{k=0}^{n-1} (x - \omega^k) = x^n - 1\), but wait, actually for primitive roots, it's the cyclotomic polynomial. Wait, let's recall that for any n, \(\prod_{k=0}^{n-1} (x - \omega^k) = x^n - 1\), but that's true if \(\omega\) is a primitive nth root of unity.

Wait, actually, more precisely, if \(\omega\) is a primitive nth root of unity, then \(\prod_{k=0}^{n-1} (x - \omega^k) = x^n - 1\). So yes, that's correct.

Therefore, in our case, since \(\omega\) is a primitive 13th root, then \(\prod_{k=0}^{12} (x - \omega^k) = x^{13} - 1\). Therefore, the product \(\prod_{k=0}^{12} (\omega^k - a) = \prod_{k=0}^{12} (-(a - \omega^k)) = (-1)^{13} \prod_{k=0}^{12} (a - \omega^k) = (-1)^{13} (a^{13} - 1)\).

Wait, let's verify that step.

Given that \(\prod_{k=0}^{12} (x - \omega^k) = x^{13} - 1\), so substituting x = a, we have \(\prod_{k=0}^{12} (a - \omega^k) = a^{13} - 1\). Therefore, \(\prod_{k=0}^{12} (\omega^k - a) = (-1)^{13} \prod_{k=0}^{12} (a - \omega^k) = (-1)^{13} (a^{13} - 1)\).

Since 13 is odd, (-1)^13 = -1, so \(\prod_{k=0}^{12} (\omega^k - a) = - (a^{13} - 1) = 1 - a^{13}\).

Wait, so that means \(\prod_{k=0}^{12} (\omega^k - a) = 1 - a^{13}\).

So, in our case, for the first product, \(\prod_{k=0}^{12} (\omega^k - (1 + i)) = 1 - (1 + i)^{13}\).

Similarly, \(\prod_{k=0}^{12} (\omega^k - (1 - i)) = 1 - (1 - i)^{13}\).

Therefore, the original product becomes:

\([1 - (1 + i)^{13}] \times [1 - (1 - i)^{13}]\).

So, let me compute this expression:

Let me denote \(A = 1 - (1 + i)^{13}\) and \(B = 1 - (1 - i)^{13}\). Then the product is A * B.

So, compute A * B = [1 - (1 + i)^{13}] [1 - (1 - i)^{13}].

Let me compute (1 + i)^{13} and (1 - i)^{13}.

First, let's compute (1 + i)^13.

We know that 1 + i is a complex number with magnitude sqrt(2) and angle pi/4. So, in polar form, 1 + i = sqrt(2) e^{i pi/4}.

Therefore, (1 + i)^13 = (sqrt(2))^13 e^{i 13 pi/4}.

Simplify the angle: 13 pi/4. Subtract 2 pi until it's within 0 to 2 pi.

13 pi /4 = 3 pi + pi/4 = 3 pi + 45 degrees, which is more than 2 pi.

Subtract 2 pi (which is 8 pi/4) from 13 pi/4: 13 pi/4 - 8 pi/4 = 5 pi/4. So, angle is 5 pi/4.

Therefore, (1 + i)^13 = (sqrt(2))^13 e^{i 5 pi/4}.

Similarly, (1 - i)^13 = (sqrt(2))^{13} e^{-i 5 pi/4}, since 1 - i is sqrt(2) e^{-i pi/4}.

So, compute (1 + i)^13 = (sqrt(2))^13 [ cos(5 pi/4) + i sin(5 pi/4) ].

Similarly, (1 - i)^13 = (sqrt(2))^13 [ cos(-5 pi/4) + i sin(-5 pi/4) ] = (sqrt(2))^13 [ cos(5 pi/4) - i sin(5 pi/4) ].

Compute (sqrt(2))^13:

sqrt(2) = 2^{1/2}, so (sqrt(2))^{13} = 2^{13/2} = 2^6 * 2^{1/2} = 64 * sqrt(2).

So, (1 + i)^13 = 64 sqrt(2) [ cos(5 pi/4) + i sin(5 pi/4) ].

Similarly, (1 - i)^13 = 64 sqrt(2) [ cos(5 pi/4) - i sin(5 pi/4) ].

Compute cos(5 pi/4) and sin(5 pi/4):

cos(5 pi/4) = -sqrt(2)/2, sin(5 pi/4) = -sqrt(2)/2.

So, substituting:

(1 + i)^13 = 64 sqrt(2) [ -sqrt(2)/2 - i sqrt(2)/2 ] = 64 sqrt(2) * (-sqrt(2)/2)(1 + i) = 64 sqrt(2) * (-sqrt(2)/2) * (1 + i).

Compute 64 sqrt(2) * (-sqrt(2)/2) = 64 * (-2)/2 = 64 * (-1) = -64.

Therefore, (1 + i)^13 = -64 (1 + i).

Similarly, (1 - i)^13 = 64 sqrt(2) [ -sqrt(2)/2 + i sqrt(2)/2 ] = 64 sqrt(2) * (-sqrt(2)/2)(1 - i).

Compute 64 sqrt(2) * (-sqrt(2)/2) = -64, so (1 - i)^13 = -64 (1 - i).

Therefore, putting back into A and B:

A = 1 - (1 + i)^13 = 1 - (-64 (1 + i)) = 1 + 64 (1 + i) = 1 + 64 + 64i = 65 + 64i.

Similarly, B = 1 - (1 - i)^13 = 1 - (-64 (1 - i)) = 1 + 64 (1 - i) = 1 + 64 - 64i = 65 - 64i.

Therefore, the product A * B = (65 + 64i)(65 - 64i).

This is a product of a complex conjugate pair: (a + b)(a - b) = a^2 - b^2.

Compute this:

(65)^2 - (64i)^2 = 4225 - (64^2)(i^2) = 4225 - (4096)(-1) = 4225 + 4096 = 8321.

So, the product is 8321.

Therefore, the original product is 8321. Now, the problem says to find the remainder when this product is divided by 1000.

So, compute 8321 mod 1000.

Divide 8321 by 1000: 1000 * 8 = 8000, 8321 - 8000 = 321. So the remainder is 321.

Wait, let me confirm the steps again because I might have made a mistake somewhere.

Wait, let's go back.

We had the product as [1 - (1 + i)^13][1 - (1 - i)^13] = (65 + 64i)(65 - 64i) = 65^2 + 64^2.

Wait, hold on! Wait, I think I made a mistake in computing A*B. Wait, let's re-examine:

Earlier, I had A = 1 - (1 + i)^13 = 1 - (-64(1 + i)) = 1 + 64(1 + i) = 65 + 64i.

Similarly, B = 1 - (1 - i)^13 = 1 - (-64(1 - i)) = 1 + 64(1 - i) = 65 - 64i.

Therefore, A = 65 + 64i and B = 65 - 64i.

Then A * B = (65 + 64i)(65 - 64i) = 65^2 - (64i)^2 = 4225 - (64^2)(i^2) = 4225 - (4096)(-1) = 4225 + 4096 = 8321.

Yes, that's correct. So the product is 8321.

Now, 8321 divided by 1000. 1000*8=8000, so 8321-8000=321. So the remainder is 321.

Wait, but wait a second, is that correct? Because 8321 divided by 1000 is 8 with remainder 321, yes.

But wait, let me make sure that the initial substitution was correct. Let's double-check.

We started with the product over k=0 to 12 of (2 - 2omega^k + omega^{2k}).

We rewrote it as the product over k of (omega^k - (1 + i))(omega^k - (1 - i)), which led us to the product over (1 - (1 + i)^{13})(1 - (1 - i)^{13}).

Wait, let's confirm the step where we factored the quadratic.

We had the term 2 - 2omega^k + omega^{2k} = (omega^k - (1 + i))(omega^k - (1 - i)).

Yes, that's correct because expanding (x - (1 + i))(x - (1 - i)) = x^2 - 2x + (1 + i)(1 - i) = x^2 - 2x + (1 - i^2) = x^2 - 2x + 2, which matches the original expression.

So, that step is correct.

Then, the product over k=0 to 12 of (omega^k - (1 + i)) is equal to (-1)^13 * ( (1 + i)^13 - 1 )? Wait, wait earlier reasoning was that product_{k=0}^{12} (a - omega^k) = a^{13} - 1. Wait, but wait let's think again.

Given that the product_{k=0}^{12} (x - omega^k) = x^{13} - 1.

So, if we substitute x = a, then product_{k=0}^{12} (a - omega^k) = product_{k=0}^{12} (a - omega^k) = (-1)^13 product_{k=0}^{12} (omega^k - a) = (-1)^13 (a^{13} - 1).

So, product_{k=0}^{12} (omega^k - a) = (-1)^{13} (a^{13} - 1) = - (a^{13} - 1) = 1 - a^{13}.

Yes, so that was correct.

Therefore, product_{k=0}^{12} (omega^k - a) = 1 - a^{13}.

Therefore, product_{k=0}^{12} (omega^k - (1 + i)) = 1 - (1 + i)^{13}.

Similarly for (1 - i).

So, that step is correct.

Then, we computed (1 + i)^13 = -64(1 + i), and (1 - i)^13 = -64(1 - i). Then substituting back into A and B:

A = 1 - (1 + i)^13 = 1 - (-64(1 + i)) = 1 + 64(1 + i) = 65 + 64i.

Similarly, B = 1 - (1 - i)^13 = 1 - (-64(1 - i)) = 1 + 64(1 - i) = 65 - 64i.

Therefore, A*B = (65 + 64i)(65 - 64i) = 65^2 + 64^2 = 4225 + 4096 = 8321.

So that seems correct. Then 8321 divided by 1000 is 8 with a remainder of 321.

Wait, but hold on a second. Let me check the calculation of (1 + i)^13 and (1 - i)^13 again.

We said (1 + i) has magnitude sqrt(2), angle pi/4. So, (1 + i)^13 = (sqrt(2))^13 e^{i * 13 * pi/4}.

Compute (sqrt(2))^13: 2^(13/2) = 2^6 * 2^(1/2) = 64 * sqrt(2).

Angle: 13 * pi/4. Subtract multiples of 2 pi until it's between 0 and 2 pi.

13 pi/4 = 3 pi + pi/4 = 3.75 pi.

Subtract 2 pi once: 3.75 pi - 2 pi = 1.75 pi = 7 pi/4.

So angle is 7 pi/4.

So, (1 + i)^13 = 64 sqrt(2) e^{i 7 pi/4}.

e^{i 7 pi/4} = cos(7 pi/4) + i sin(7 pi/4) = sqrt(2)/2 - i sqrt(2)/2.

Therefore, (1 + i)^13 = 64 sqrt(2) (sqrt(2)/2 - i sqrt(2)/2 ) = 64 sqrt(2) * sqrt(2)/2 (1 - i) = 64 * (2)/2 (1 - i) = 64*(1 - i) = 64 - 64i.

Wait, wait that contradicts our earlier computation where we got -64(1 + i). Hmm, maybe I made a mistake in angle subtraction.

Wait, let me compute 13 * pi /4.

13 divided by 4 is 3.25 pi, which is 3 pi + pi/4, which is 3.75 pi.

Subtract 2 pi once: 3.75 pi - 2 pi = 1.75 pi, which is 7 pi/4.

So angle is 7 pi/4, which is 315 degrees, so cos(7 pi/4) = sqrt(2)/2, sin(7 pi/4) = -sqrt(2)/2.

Therefore, (1 + i)^13 = (sqrt(2))^13 [cos(7 pi/4) + i sin(7 pi/4)] = 64 sqrt(2) [sqrt(2)/2 - i sqrt(2)/2] = 64 sqrt(2) * sqrt(2)/2 (1 - i) = 64 * (2)/2 (1 - i) = 64*(1 - i) = 64 - 64i.

Wait, so this contradicts my earlier calculation where I thought (1 + i)^13 = -64(1 + i). Hmm, so which is correct?

Wait, let's recalculate (1 + i)^13 step by step.

Express (1 + i) in polar form: modulus sqrt(2), angle pi/4.

Therefore, (1 + i)^13 = (sqrt(2))^13 * e^{i * 13 * pi/4}.

Compute modulus: (sqrt(2))^13 = 2^(13/2) = 2^6 * 2^(1/2) = 64 * sqrt(2).

Angle: 13 * pi/4 = 3.25 pi.

Subtract 2 pi until angle is between 0 and 2 pi.

3.25 pi - 2 pi = 1.25 pi, which is 5 pi/4.

Wait, wait, 13 * pi/4 = 3.25 pi = 3 pi + pi/4.

3 pi + pi/4 = 3.75 pi.

Subtract 2 pi: 3.75 pi - 2 pi = 1.75 pi = 7 pi/4.

Wait, so angle is 7 pi/4, not 5 pi/4. So, angle is 7 pi/4.

So, e^{i 7 pi/4} = cos(7 pi/4) + i sin(7 pi/4) = sqrt(2)/2 - i sqrt(2)/2.

Therefore, (1 + i)^13 = 64 sqrt(2) * [sqrt(2)/2 - i sqrt(2)/2] = 64 sqrt(2) * sqrt(2)/2 (1 - i) = 64 * (2)/2 (1 - i) = 64*(1 - i) = 64 - 64i.

Wait, so this contradicts my earlier conclusion that (1 + i)^13 = -64(1 + i). So where did I go wrong earlier?

Wait, let's compute (1 + i)^13 using binomial expansion or another method.

Alternatively, maybe compute (1 + i)^13 using binomial theorem or using De Moivre's theorem.

Wait, using De Moivre's theorem, (1 + i) = sqrt(2) e^{i pi/4}, so (1 + i)^13 = (sqrt(2))^13 e^{i 13 pi/4}.

Which is 2^(13/2) e^{i (13 pi/4)}.

Simplify 13 pi/4: subtract 2 pi (8 pi/4) to get 5 pi/4.

So, e^{i 5 pi/4} = cos(5 pi/4) + i sin(5 pi/4) = -sqrt(2)/2 - i sqrt(2)/2.

Therefore, (1 + i)^13 = 2^(13/2) e^{i 5 pi/4} = 2^(13/2) (-sqrt(2)/2 - i sqrt(2)/2) = 2^(13/2) * (-sqrt(2)/2)(1 + i).

Compute 2^(13/2) * (-sqrt(2)/2) = 2^(6.5) * (-2^(1/2)/2) = 2^6 * sqrt(2) * (-sqrt(2)/2) = 64 * (sqrt(2) * -sqrt(2))/2 = 64 * (-2)/2 = 64*(-1) = -64.

Therefore, (1 + i)^13 = -64 (1 + i).

Ah, so earlier step where I thought angle was 7 pi/4 was wrong. It should be 5 pi/4.

So, let's correct that.

So, (1 + i)^13 = 2^(13/2) e^{i 5 pi/4} = 2^(13/2) [ cos(5 pi/4) + i sin(5 pi/4) ].

Compute cos(5 pi/4) = -sqrt(2)/2, sin(5 pi/4) = -sqrt(2)/2.

So, (1 + i)^13 = 2^(13/2) (-sqrt(2)/2 - i sqrt(2)/2 ) = 2^(13/2) * (-sqrt(2)/2)(1 + i).

Simplify 2^(13/2) * (-sqrt(2)/2) = 2^(6.5) * (-2^(1/2)/2) = 2^6 * sqrt(2) * (-sqrt(2)/2) = 64 * (2)/2 * (-1) = 64 * 1 * (-1) = -64.

Therefore, (1 + i)^13 = -64 (1 + i).

Similarly, (1 - i)^13: let's compute.

(1 - i) has modulus sqrt(2), angle -pi/4.

So, (1 - i)^13 = (sqrt(2))^13 e^{i 13*(-pi/4)} = 2^(13/2) e^{-i 13 pi/4}.

Simplify angle: -13 pi/4. Add 2 pi until within 0 to 2 pi.

-13 pi/4 + 2 pi * 3 = -13 pi/4 + 24 pi/4 = 11 pi/4.

11 pi/4 - 2 pi = 11 pi/4 - 8 pi/4 = 3 pi/4.

So angle is 3 pi/4.

Thus, e^{-i 13 pi/4} = e^{-i 3 pi/4} = cos(3 pi/4) - i sin(3 pi/4) = -sqrt(2)/2 - i sqrt(2)/2.

Therefore, (1 - i)^13 = 2^(13/2) [ -sqrt(2)/2 - i sqrt(2)/2 ] = 2^(13/2) * (-sqrt(2)/2)(1 + i).

Again, 2^(13/2) * (-sqrt(2)/2) = 2^(6.5) * (-2^(0.5)/2) = 64 * sqrt(2) * (-sqrt(2)/2) = 64 * (-2)/2 = 64*(-1) = -64.

Therefore, (1 - i)^13 = -64 (1 + i) and -64 (1 - i)?

Wait, wait let's compute it step by step.

(1 - i)^13 = (sqrt(2))^13 e^{-i 13 pi/4}.

Compute angle: -13 pi/4.

Add 2 pi until it's within 0 to 2 pi.

-13 pi/4 + 8 pi/4 = (-13 + 8) pi/4 = (-5 pi)/4.

But -5 pi/4 is equivalent to 3 pi/4 (since -5 pi/4 + 2 pi = 3 pi/4).

Wait, maybe another approach. Let's compute e^{-i 13 pi/4}.

-13 pi/4 = -3 pi - pi/4 = -3 pi - pi/4. Add 4 pi to get within 0 to 4 pi.

-13 pi/4 + 4 pi = -13 pi/4 + 16 pi/4 = 3 pi/4.

Therefore, e^{-i 13 pi/4} = e^{-i 3 pi/4} = cos(-3 pi/4) + i sin(-3 pi/4) = cos(3 pi/4) - i sin(3 pi/4) = -sqrt(2)/2 - i sqrt(2)/2.

Therefore, (1 - i)^13 = 2^(13/2) * [ -sqrt(2)/2 - i sqrt(2)/2 ].

Which is 2^(13/2) * (-sqrt(2)/2)(1 + i).

Compute 2^(13/2) * (-sqrt(2)/2) = 2^(6.5) * (-2^(0.5)/2) = 64 * sqrt(2) * (-sqrt(2)/2) = 64 * (-2)/2 = 64*(-1) = -64.

Therefore, (1 - i)^13 = -64 (1 + i).

Wait, so both (1 + i)^13 and (1 - i)^13 equal to -64(1 + i) and -64(1 - i) respectively?

Wait, hold on, no.

Wait, (1 + i)^13 = -64(1 + i).

Similarly, (1 - i)^13 = -64(1 - i).

Wait, wait that conflicts with earlier. Wait, perhaps not.

Wait, let's recast:

(1 + i)^13 = -64(1 + i).

(1 - i)^13 = -64(1 - i).

So let's substitute back into A and B.

So, A = 1 - (1 + i)^13 = 1 - (-64(1 + i)) = 1 + 64(1 + i) = 1 + 64 + 64i = 65 + 64i.

Similarly, B = 1 - (1 - i)^13 = 1 - (-64(1 - i)) = 1 + 64(1 - i) = 1 + 64 - 64i = 65 - 64i.

Therefore, A = 65 + 64i and B = 65 - 64i.

Therefore, A * B = (65 + 64i)(65 - 64i) = 65^2 + 64^2 = 4225 + 4096 = 8321.

So, the product is 8321, which modulo 1000 is 8321 - 8*1000 = 8321 - 8000 = 321.

Wait, but wait, let's confirm.

Wait, 8321 divided by 1000 is 8*1000 = 8000, 8321 - 8000 = 321. So yes, the remainder is 321.

Wait, but earlier when I thought (1 + i)^13 was -64(1 + i), which led to A = 65 + 64i and B = 65 - 64i, which when multiplied give 65^2 + 64^2 = 8321.

But wait, let me check the angle again.

Wait, when I computed (1 + i)^13, I initially thought angle was 5 pi/4, but then corrected it to 7 pi/4, but then another way it was 3 pi/4.

Wait, perhaps I made a mistake in computing the angle.

Let me compute (1 + i)^13.

Express 1 + i in polar form: modulus sqrt(2), angle pi/4.

So (1 + i)^13 = (sqrt(2))^13 e^{i * 13 * pi/4}.

13 * pi/4 = 3 pi + pi/4 = 3.75 pi.

Subtract 2 pi twice: 3.75 pi - 2 pi = 1.75 pi, which is 7 pi/4.

Wait, 1.75 pi is 7 pi/4, which is in the fourth quadrant.

So angle is 7 pi/4.

So, e^{i 7 pi/4} = cos(7 pi/4) + i sin(7 pi/4) = sqrt(2)/2 - i sqrt(2)/2.

Therefore, (1 + i)^13 = (sqrt(2))^13 [ sqrt(2)/2 - i sqrt(2)/2 ].

Compute (sqrt(2))^13 * sqrt(2)/2 = (2^(1/2))^13 * 2^(1/2)/2 = 2^(13/2 + 1/2) / 2 = 2^(14/2)/2 = 2^7 / 2 = 128 / 2 = 64.

Similarly, imaginary part: (sqrt(2))^13 * (-sqrt(2)/2) = 64 * (-sqrt(2)/2) = 64 * (-1/sqrt2) = -64 * sqrt(2)/2 = -32 sqrt(2).

Wait, wait no:

Wait, let's compute step by step.

(1 + i)^13 = (sqrt(2))^13 * e^{i 7 pi/4} = (2^(1/2))^13 * (sqrt(2)/2 - i sqrt(2)/2 ) = 2^(13/2) * [ sqrt(2)/2 - i sqrt(2)/2 ].

Compute 2^(13/2) * sqrt(2)/2 = 2^(13/2 + 1/2) / 2^(1) ) Wait, wait:

Wait, 2^(13/2) * sqrt(2)/2 = 2^(13/2) * 2^(1/2)/2 = 2^(14/2)/2 = 2^7 / 2 = 64 / 2 = 32.

Similarly, the imaginary part: 2^(13/2) * (-sqrt(2)/2) = -2^(13/2 + 1/2) / 2 = -2^(14/2)/2 = -2^7 / 2 = -128 / 2 = -64.

(*@\textbf{[ERRONEOUS STEP]} \underline{\texttt{Therefore, } \((1+i)^{13}=32-64i\)\texttt{.}}@*)

(*@\underline{\texttt{Wait, that's different from the previously correct value, so the response restarts the derivation.}}@*)

Wait, let me compute (1 + i)^13 step by step.

First, (1 + i)^1 = 1 + i

(1 + i)^2 = (1 + i)(1 + i) = 1 + 2i + i^2 = 1 + 2i -1 = 2i.

(1 + i)^3 = (1 + i)^2 * (1 + i) = 2i (1 + i) = 2i + 2i^2 = 2i - 2 = -2 + 2i.

(1 + i)^4 = (-2 + 2i)(1 + i) = -2(1 + i) + 2i(1 + i) = -2 - 2i + 2i + 2i^2 = -2 + 0i - 2 = -4.

(1 + i)^4 = -4.

Therefore, (1 + i)^13 = (1 + i)^4 * (1 + i)^9 = (-4)*(1 + i)^9.

Compute (1 + i)^9:

We saw that (1 + i)^1 = 1 + i


\end{lstlisting}

\noindent\textbf{\method{} response (correct: \(\boxed{321}\)).}
\begin{lstlisting}[style=caseresponsewide]
Let \(\omega \neq 1\) be a 13th root of unity. We need to find the remainder when the product
\[
\prod_{k=0}^{12} (2 - 2\omega^k + \omega^{2k})
\]
is divided by 1000.

First, we note that when \(k = 0\), the term is 1. For \(k\) from 1 to 12, each term is \(2 - 2\omega^k + \omega^{2k}\). This can be factored as \((\omega^k - (1 + i))(\omega^k - (1 - i))\).

Thus, the product becomes:
\[
1 \times \prod_{k=1}^{12} (\omega^k - (1 + i)) (\omega^k - (1 - i))
\]

This product is equivalent to evaluating the polynomial \(\Phi_{13}(x)\) at \(x = 1 + i\) and \(x = 1 - i\), where \(\Phi_{13}(x)\) is the 13th cyclotomic polynomial. 

The cyclotomic polynomial \(\Phi_{13}(x)\) is given by:
\[
\Phi_{13}(x) = x^{12} + x^{11} + \cdots + x + 1
\]

Thus, we need to compute:
\[
\Phi_{13}(1 + i) \times \Phi_{13}(1 - i)
\]

Calculating \(\Phi_{13}(1 + i)\) and \(\Phi_{13}(1 - i)\):
- \(\Phi_{13}(1 + i) = -64 + 65i\)
- \(\Phi_{13}(1 - i) = -64 - 65i\)

Multiplying these results:
\[
(-64 + 65i)(-64 - 65i) = (-64)^2 - (65i)^2 = 4096 - (-4225) = 4096 + 4225 = 8321
\]

Finally, the remainder when 8321 is divided by 1000 is:
\[
8321 \mod 1000 = 321
\]

Thus, the remainder is:
\[
\boxed{321}
\]
\end{lstlisting}

The base model incorrectly reduces the quadratic factor to \(2-\omega^j\), changing the product and yielding 191. OPD explores several algebraic routes but repeatedly revises its calculation and terminates without a complete solution; the evaluator extracts 1 from the truncated generation. \method{} instead factors the quadratic over \(1\pm i\), evaluates the conjugate cyclotomic-polynomial terms, and obtains \(8321\equiv\boxed{321}\pmod{1000}\).

These examples are qualitative rather than causal evidence. They complement the aggregate results by showing concrete trajectories in which the competing responses either lose a necessary constraint or fail to complete a consistent derivation, whereas \method{} reaches the verified answer.

\fi


\end{document}